\documentclass[12pt,twoside]{article}

\usepackage{graphicx} 
\usepackage{subfigure} 

\usepackage{tikz}

\usepackage{algorithm}
\usepackage{algorithmic}

\usepackage{hyperref}

\newtheorem{theorem}{Theorem}
\newtheorem{lemma}{Lemma}
\newtheorem{corollary}{Corollary}
\newtheorem{proposition}[theorem]{Proposition}

\newenvironment{keywords}{\centerline{\bf\small
Keywords}\begin{quote}\small}{\par\end{quote}\vskip 1ex}

\newenvironment{proof}[1][Proof]{\begin{trivlist}
\item[\hskip \labelsep {\bfseries #1}]}{\end{trivlist}}
\newenvironment{definition}[1][Definition]{\begin{trivlist}
\item[\hskip \labelsep {\bfseries #1}]}{\end{trivlist}}

\newenvironment{remark}[1][Remark]{\begin{trivlist}
\item[\hskip \labelsep {\bfseries #1}]}{\end{trivlist}}

\makeatletter
\newcommand{\figcaption}{\def\@captype{figure}\caption}
\newcommand{\tabcaption}{\def\@captype{table}\caption}
\makeatother

\newcommand{\qed}{\nobreak \ifvmode \relax \else
      \ifdim\lastskip<1.5em \hskip-\lastskip
      \hskip1.5em plus0em minus0.5em \fi \nobreak
      \vrule height0.75em width0.5em depth0.25em\fi}

\usepackage{amssymb}
\usepackage{amsmath}

\begin{document} 
\title{\vspace{-4ex}
\vskip 2mm\bf\Large\hrule height5pt \vskip 4mm
Theory of Dependent Hierarchical Normalized Random Measures
\vskip 4mm \hrule height2pt}

\author{{\bf Changyou Chen}$^{1,2}$\\{\tt cchangyou@gmail.com}\\
$^1$RSISE, The Australian National University,\\
 Canberra, ACT, Australia\\
\and
{\bf Wray Buntine}$^{2,1}$\\{\tt Wray.Buntine@nicta.com.au}\\
$^2$NICTA, ACT, Australia
\and
{\bf Nan Ding}$^3$\\{\tt ding10@purdue.edu}\\
$^3$Department of Computer Science, Purdue University, USA\\
}

\maketitle

\newpage

\begin{abstract} 
This paper presents theory for
Normalized Random Measures (NRMs),
Normalized Generalized Gammas (NGGs),
a particular kind of NRM,
and Dependent Hierarchical NRMs
which allow networks of dependent NRMs to be
analysed.
These have been used, for instance,
for time-dependent topic modelling.
In this paper,
we first introduce some mathematical 
background of completely random measures (CRMs)
and their construction from Poisson processes,
and then introduce NRMs and NGGs.
Slice sampling is also introduced for posterior inference.
The dependency
operators in Poisson processes and 
for the corresponding CRMs and NRMs is then introduced
and Posterior inference for the NGG presented.
Finally,  we give dependency  and composition results
when applying these operators to NRMs so they can be used
in a network with hierarchical and dependent relations.
\end{abstract} 


\newpage
\tableofcontents

\newpage

\section{Introduction}

This paper presents theory for
Normalized Random Measures (NRMs),
Normalized Generalized Gammas (NGGs),
a particular kind of NRM,
and Dependent Hierarchical NRMs
which allow networks of dependent NRMs to be
analysed.
These have been used, for instance,
for time-dependent topic modelling \cite{ChenDB:ICML12}.

Dependency models are getting more and more popular in machine
learning recently due to the fact of correlated data we are facing at,
{\it e.g.}, real data is always correlated with each other rather than
independent.  The pioneer work of
MachEachern~\cite{MacEachern:Bayes99,MacEachern:DDP00} 
treats the jumps and atoms to be stochastic between dependent models.
While there are many ways of constructing dependent nonparametric
models, {\it e.g.}, from a stick-breaking
construction~\cite{GriffinS:09}, or from a hierarchical
construction~\cite{TehJorBea2006}, in this paper, following the idea
of~\cite{LinGF:NIPS10}, we construct dependency normalized random
measures from the underlying Poisson processes of the corresponding
completely random measures~\cite{Kingman:PJM67}.  This construction is
intuitive and allow flexibly controlling of the dependencies.  
A related construction in the statistical literature
is by Lijoi {\it et al.}~\cite{LijoiNP:12} 
that deals with modeling two groups of data.

In this paper,
we first introduce in Section~\ref{sec:background} some mathematical 
background of completely random measures (CRMs)
and their construction from Poisson processes,
and then introduce NRMs and NGGs.
Slice sampling is also introduced to do the posterior 
sampling of NRMs using techniques from \cite{GriffinW:JCGS11}.
The dependency
operators in Poisson processes and 
for the corresponding CRMs and NRMs is then introduced in
Section~\ref{sec:ops}
following the work of \cite{Kindman:PoiP10,LinGF:NIPS10}.
Posterior inference for the NGG are then developed in
Section~\ref{sec:postngg} based on the results of
\cite{JamesLP:SJS09}.
Then we give the dependency  and composition results
when applying these operators to NRMs in Section~\ref{sec:depprop}.
Proofs are given in 
the Appendix, Section~\ref{sec:proof}.

\section{Background}\label{sec:background}

In this section we briefly introduce background of Poisson processes,
the corresponding completely random measures, 
dependency operations on these random measures,
and normalized random measures.

Section~\ref{sec:PoiP2RM} explains how to construct
completely random measures from Poisson processes. 
Section~\ref{sec:PoiPOP} introduces
operations on Poisson processes to construct dependent Poisson processes.
Section~\ref{sec:URMPOP}
adapts these operations to the corresponding completely random measures 
(CRMs).
Constructing normalized random measures (NRMs) from CRMs is discussed
in Section~\ref{sec:NRM} along with
details of the Normalized Generalized Gamma (NGG),
a particular kind of NRM for which the details have been worked out. 
A slice sampler for sampling an NRM is described 
in Section~\ref{sec:SNRM}.

We first give an illustration of 
the basic construction for an NRM.
for a target domain $\mathbb{X}$.
The Poisson process is used to create a countable (and usually) infinite
set points in a product space of $\mathbb{R}^+$ with the target domain
$\mathbb{X}$, as shown in the left of Figure~\ref{fig:example}.
The distribution is then a discrete one on these
points. The distribution can be pictured by dropping lines
from each point $(t,x)$ down to $(0,x)$, and then
normalizing all these lines so their sum is one.  The resulting
picture shows the set of weighted impulses that make up the
constructed NRM on the target domain.

\begin{center}
   \centering
     \begin{minipage}{0.47\linewidth}
        \centering
        \includegraphics[width=0.9\linewidth]{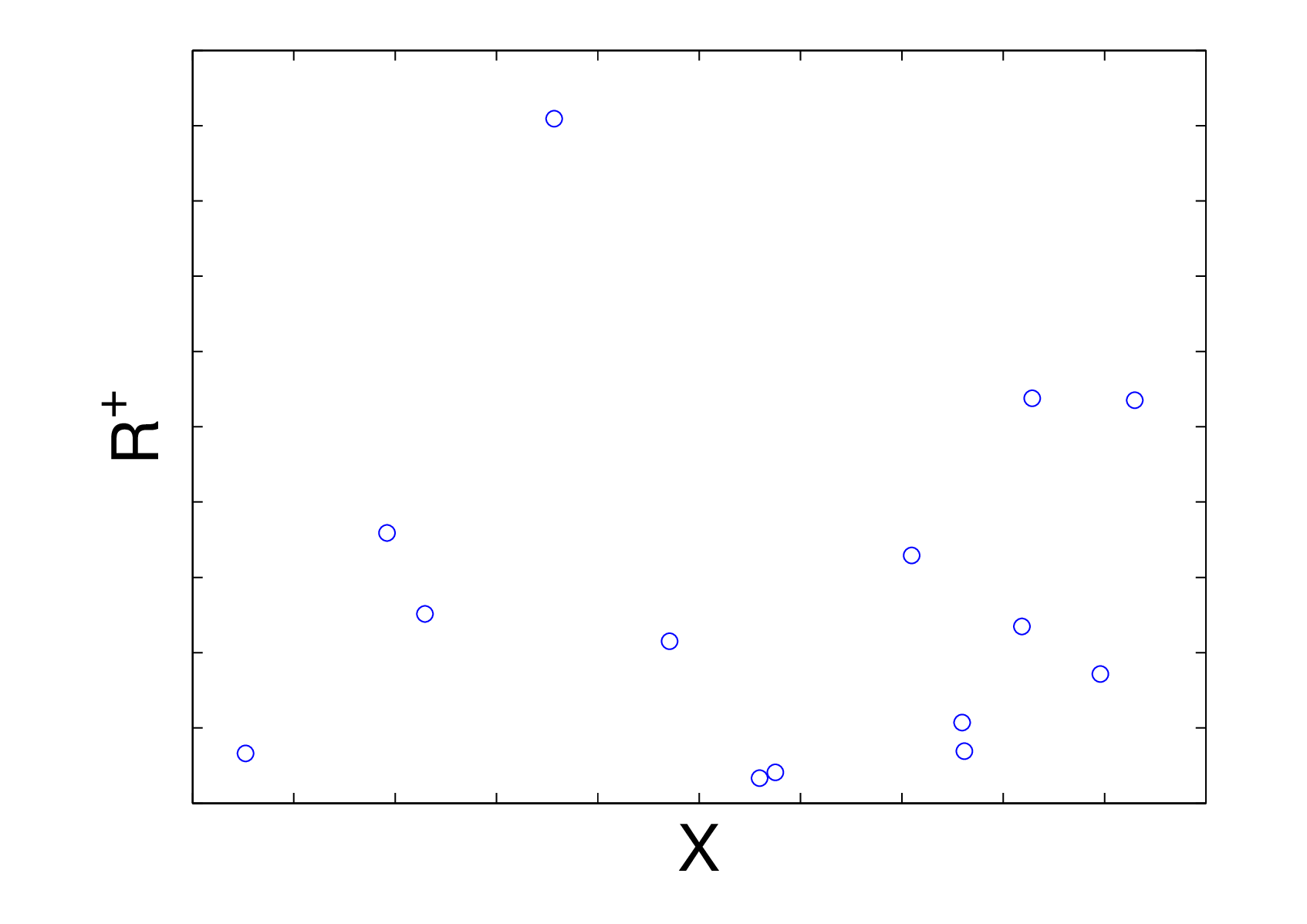}
	\\{\small\centerline{\bf Counting process:}}\medskip
	{\small\centerline{\bf $N(\cdot) = \sum_{k}\delta_{(J_k,x_k)}(\cdot)$}}\medskip
    \end{minipage}
    \begin{minipage}{0.04\linewidth}
    $\rightarrow$
    \end{minipage}
    \begin{minipage}{0.47\linewidth}
        \centering
        \includegraphics[width=0.9\linewidth]{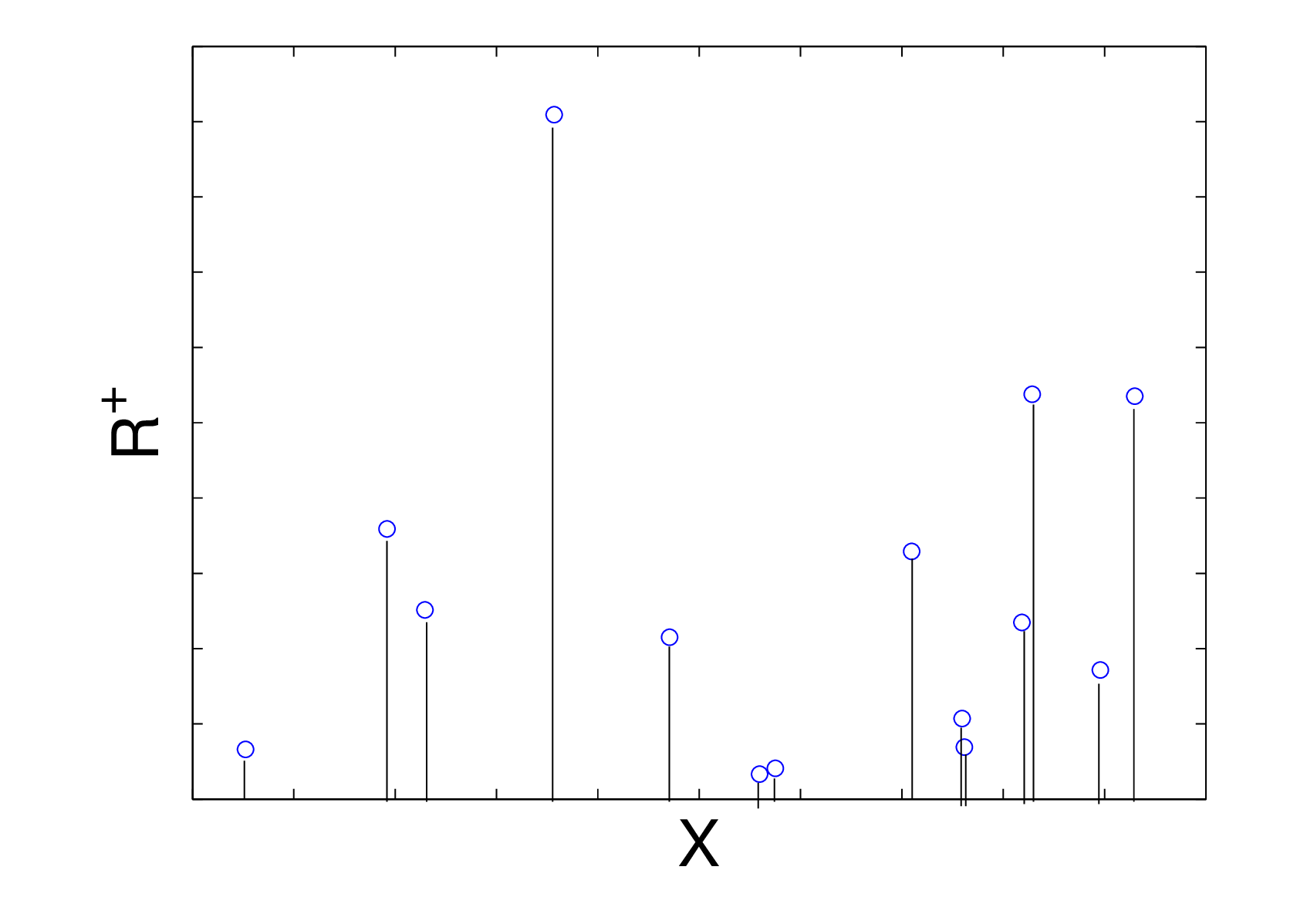}
	\hspace{2cm}{\small\centerline{\bf Completely random measure:}}\medskip
	\hspace{2cm}{\small\centerline{\bf $\tilde{\mu}(\cdot) = \sum_{k}J_k\delta_{x_k}(\cdot)$}}\medskip
    \end{minipage}
  \figcaption{Constructing a completely random measure from a
counting process $N(\cdot)$ with points at $(J_k,x_k)$.} \label{fig:example}
\end{center}

\subsection{Constructing Completely Random Measures from Poisson processes}\label{sec:PoiP2RM}

In contrast to the general class of completely random measure (CRM)~\cite{Kingman:PJM67}, 
which admits a unique decomposition as the summation over there parts: a deterministic
measure, a purely atomic measure with fixed atom locations and a measure\footnote{can be continuous
or discrete.} with random jumps
and atoms, in this paper, we restrict it to the class of 
pure jump processes~\cite{FergusonK:AMS72}, 
which has the following form
\begin{equation}\label{eq:rmdef}
 \tilde{\mu} = \sum_{k = 1}^\infty J_k\delta_{x_k},
\end{equation}
where $J_1, J_2, \cdots > 0$ are called the jumps of the process, 
and $x_1, x_2, \cdots$ are a sequence of independent random variables
drawn from a base measurable space 
$(\mathbb{X}, \mathcal{B}(\mathbb{X}))$\footnote{$\mathcal{B}(\mathbb{X})$
means the $\sigma$-algebra of $\mathbb{X}$, we sometimes omit this and use
$\mathbb{X}$ to denote the measurable space.}.

It is shown that these kinds of CRMs can be 
constructed from Poisson processes with specific mean 
measures $\nu(\cdot)$. We will start from some definitions. 

\begin{definition}[Poisson Distributions:]
 A random variable $X$ taking values in $\mathbb{N} = \{0, 1, \cdots, \infty\}$ is said to have
  the Poisson distribution with mean $c$ in $(0, \infty)$ if
  \begin{equation}
   p(X = k|c) = \frac{e^{-c}c^k}{k!}, k\in\mathbb{N},
  \end{equation}
  then $X < \infty$ almost surely and $\mathbb{E}[X] = \text{Var}[X] = c$.
\end{definition}

\begin{definition}[Poisson Processes:]
 Let $(\mathbb{S}, \mathcal{S})$ be a measure space where $\mathcal{S}$ is the
 $\sigma$-algebra of $\mathbb{S}$. Let $\nu(\cdot)$ be a 
  measure on it. A Poisson process on $\mathbb{S}$ is defined to be a 
  random subset $\Pi\in \mathbb{S}$ such that if $N(A)$ is the number 
  of points of $\Pi$ in the measurable subset $A\subseteq \mathbb{S}$, then
  \begin{itemize}
   \item[a) ] $N(A)$ is a random variable having the Poisson 
      distribution with mean $\nu(A)$, and
   \item[b) ] whenever $A_1, \cdots, A_n$ are in $\mathcal{S}$ 
      and disjoint, the random variables $N(A_1), \cdots, N(A_n)$ 
      are independent.
  \end{itemize}
  The integer-value random measure $N(\cdot)$ is called a
  {\em Poisson random measure} and the Poisson process is denoted
  as $\Pi\sim\text{PoissonP}(\nu)$, where $\nu$ is called the mean measure
  of the Poisson process.
\end{definition}

\begin{definition}[Completely Random Measure:]
In this paper, we define a random measure on 
$(\mathbb{X}, \mathcal{B}(\mathbb{X}))$ to be a linear functional 
of the Poisson random measure $N(\cdot)$, whose mean measure
$\nu(\mathrm{d}t, \mathrm{d}x)$ defined on a product space
$\mathbb{S} = R^+\times\mathbb{X}$:
\begin{equation}\label{eq:RM}
 \tilde{\mu}(B) = \int_{\mathbb{R}^+\times B}tN(\mathrm{d}t, \mathrm{d}x), 
    \forall B\in\mathcal{B}(\mathbb{X}).
\end{equation}
The mean measure $\nu(\mathrm{d}t, \mathrm{d}x)$ 
is called the {\em L\'{e}vy measure} of $\tilde{\mu}$.
\end{definition}
The general treatment of constructing random measures from Poisson random measures
can be found in~\cite{James:AST05}.
Note that the random measure $\tilde{\mu}$ in construction (\ref{eq:RM}) has the same form
as Equation~(\ref{eq:rmdef}) because $N(\cdot)$ is composed of a countable number of points.
It can be proven to be a completely 
random measure~\cite{Kingman:PJM67} on $\mathbb{X}$, meaning that for
arbitrary disjoint subsets $\{A_i\in\mathbb{X}\}$ of the measurable space, 
the random variables $\{\tilde{\mu}(A_i)\}$ are independent.

For the completely random measure defined above to always be finite,
it is necessary that 
$\int_{\mathbb{R}^+\times \mathbb{X}}t \,\nu(\mathrm{d}t, \mathrm{d}x)$
be finite,
and therefore for every $z>0$,
$\nu([z,\infty)\times \mathbb{X})=\int_{z}^\infty \int_{\mathbb{X}} \nu(\mathrm{d}t, \mathrm{d}x)$
is finite~\cite{Kindman:PoiP10}.
It follows that there will always be a finite number of points
with jumps $J_k>z$ for that $z>0$.
Therefore in the bounded product space
$[z,\infty)\times \mathbb{X}$ the measure $\nu(\mathrm{d}t, \mathrm{d}x)$
is finite.  So it is meaningful to sample those points
$(J_k,x_k)$ with $J_k>z$ by first getting the
count of points $K$ sampled from a Poisson with (finite) mean
$\nu([z,\infty)\times \mathbb{X})$,
and then to sample the $K$ points according to the
distribution of
$\frac{\nu(\mathrm{d}t, \mathrm{d}x)}{\nu([z,\infty)\times \mathbb{X})}$.

Without loss of generality, the L\'{e}vy measure of
Equation~(\ref{eq:RM}) can be represented as
$\nu(\mathrm{d}t, \mathrm{d}x) = M\rho_{\eta}(\mathrm{d}t|x)H(\mathrm{d}x)$, 
where $\eta$ denotes
the hyper-parameters if any of a measure on $t$, 
$H(\mathrm{d}x)$ is a probability measure so
$H(\mathbb{X})=1$, and $M$ is called the {\it mass} of the
L\'evy measure.  
Note the total measure of
$\rho_{\eta}(\mathrm{d}t|x)$ is not standardized in any way
so in principle some mass could also appear in $\rho_{\eta}(\mathrm{d}t|x)$.
The mass is used as a concentration parameter for the random measure.

A realization of $\tilde{\mu}$ on $\mathbb{X}$
can be constructed by sampling from the underlying Poisson process 
in a number of ways, either in rounds for
decreasing bounds $z$ using the logic just given,
or by explicitly sampling
the jumps in order.  The later goes as
follows~\cite{FergusonK:AMS72}:
\begin{lemma}[Sampling a CRM]\label{lem:sampler}
Sample a CRM $\tilde{\mu}$ with L\'evy measure
$\nu(\mathrm{d}t, \mathrm{d}x) = M\rho_{\eta}(\mathrm{d}t|x)H(\mathrm{d}x)$
as follows.
\begin{itemize}
  \item Draw {\it i.i.d.}~samples $x_i$ from the base measure $H(\mathrm{d}x)$.
  \item Draw the corresponding weights $J_i$ for these {\it i.i.d.} samples in 
	decreasing order, which goes as:
      \begin{itemize}
	\item Draw the largest jump $J_1$ from the cumulative distribution function 
	      $P(J_1\leq j_1) = \exp\left\{-M\!\int_{j_1}^\infty
		\rho_\nu (\mathrm{d}t|x_i)\right\}$.
	\item Draw the second largest jump $J_2$ from the cumulative distribution function 
	      $P(J_2\leq j_2) = \exp\left\{-M\int_{j_2}^{j_1}
		\rho_\nu(\mathrm{d}t|x_2)\right\}$.
	\item $\cdots$
      \end{itemize}
  \item The random measure $\tilde{\mu}$ then can now be realized as $\tilde{\mu} = \sum_iJ_i\delta_{x_i}$.
\end{itemize}
\end{lemma}

As a random variable is uniquely determined by its Laplace transformation, 
the random measure $\tilde{\mu}$ is uniquely characterized by its 
{\em Laplace functional} through the 
L\'{e}vy-Khintchine representation of a L\'{e}vy process~\cite{Cinlar:10}.
That is, for any measurable function $f: \mathbb{X}\rightarrow\mathbb{R}^+$, 
we have
\begin{eqnarray}\label{eq:levy-khintchine}
 &&\mathbb{E}\left[\exp\left\{-\int_{\mathbb{X}}f(x)
      \tilde{\mu}(\mathrm{d}x)\right\}\right] 
  = \exp\left\{-\int_{\mathbb{R}^+ \times \mathbb{X}}\left[1 - 
      \exp\left\{-tf(x)\right\}\right]\nu(\mathrm{d}t, \mathrm{d}x)\right\},
\end{eqnarray}

Now instead of dealing with $\tilde{\mu}$ itself, we deal with $\nu(\mathrm{d}t, \mathrm{d}x)$, which is
called the L\'{e}vy measure of $\tilde{\mu}$,
whose role in generating the measure via a Poisson process
was explained above.

In the case where the measure on the jumps is not dependent
on the data $x$, so  $\rho_{\eta}(\mathrm{d}t|x) = \rho_{\eta}(\mathrm{d}t)$, 
then $\tilde{\mu}$ is called homogeneous, which is the case considered
in this paper. 
When $f$ does not depend on $x$, 
(\ref{eq:levy-khintchine}) simplifies to
\begin{equation}\label{eq:levy-khintchine1}
\mathbb{E}\left[\exp\left\{-f\,\tilde{\mu}(B)\right\}\right] 
  =
 \exp\left\{-M\,p(B)\,\int_{\mathbb{R}^+}\left[1 - 
      \exp\left\{-tf\right\}\right]\rho_{\eta}(\mathrm{d}t)\right\}~.
\end{equation}
Note the term inside the exponential plays an important role
in subsequent theory, so it is given a name.
\begin{definition}[Laplace exponent:]
The {\em Laplace exponent},
denoted as $\psi_{\eta}(f)$ for a CRM with parameters $\eta$ is
given by
\begin{eqnarray}
\psi_{\eta}(f) \nonumber
&=& \int_{\mathbb{R}^+ \times \mathbb{X}}\left[1 - 
      \exp\left\{-tf\right\}\right]\nu(\mathrm{d}t, \mathrm{d}x)\\
\label{eq:laplace-exponent}
&=& M \int_{\mathbb{R}^+}\left[1 - 
      \exp\left\{-tf\right\}\right]\rho_{\eta}(\mathrm{d}t)~
~~~~~~~~~~~~~~\mbox{(homogeneous case)}~.
\end{eqnarray}
\end{definition}

Note that to guarantee the positiveness of jumps in the random measure, $\rho(\mathrm{d}t)$ in 
the L\'{e}vy measure should satisfy $\int_0^\infty\rho_\eta(\mathrm{d}t) = +\infty$~\cite{Cinlar:10},
which leads to the following equations:
\begin{equation}\label{eq:L0}
 \psi_\eta(0) = 0, \hspace{1cm}\psi_\eta(+\infty) = +\infty.
\end{equation}
That $\psi_\eta(f)$ is finite for finite positive $f$ implies 
(or is a consequence of) 
$\int_0^\infty t \rho_\eta(\mathrm{d}t)$ being finite.

\begin{remark}
There are thus four different ways to define or interpret a CRM:
\begin{enumerate}
\item
via the linear functional of Equation~(\ref{eq:RM}),
\item
through the 
L\'{e}vy-Khintchine representation of Equation~(\ref{eq:levy-khintchine})
using the Laplace exponent,
\item
sampling in order of decreasing jumps using Lemma~\ref{lem:sampler}, and
\item
sampling in blocks of decreasing jump values 
as discussed before Lemma~\ref{lem:sampler}.
\end{enumerate}
\end{remark}

\subsection{Normalized random measures}\label{sec:NRM}

\begin{definition}[Normalized Random Measures (NRM)]
 Based on (\ref{eq:RM}), a normalized random measure on 
  $(\mathbb{X}, \mathcal{B}(\mathbb{X}))$ is defined as\footnote{In this paper, we use $\mu$ to denote a normalized
   random measure, while use $\tilde{\mu}$ to denote its
    unnormalized counterpart.}
  \begin{equation}\label{eq:NRM}
   \mu = \frac{\tilde{\mu}}{\tilde{\mu}(\mathbb{X})}.
  \end{equation}
\end{definition}

The original idea of constructing random probabilities by normalizing
completely random measures on $\mathbb{R}$, namely increasing additive processes,
can be found in~\cite{EgazziniLP:AST03}, where it is termed normalized random measures
with independent increment (NRMI) and the existence of such random measures is proved.
This idea can be easily generalized from $\mathbb{R}$ to any parameter space $\mathbb{X}$,
{\it e.g.}, $\mathbb{X}$ being the Dirichlet distribution space in topic modeling.
Also note that the idea of normalized random measures can be taken as doing a transformation
$Tr(\cdot)$ on completely random measures, that is $\mu = Tr(\tilde{\mu})$. In the normalized random measure 
case, $Tr(\cdot)$ is a transformation such that $Tr(\tilde{\mu}(\mathbb{X})) = 1$.
A concise survey of other kinds of transformations can be found in~\cite{LijoiP:beyondDP09}.

Taking different L\'{e}vy measures $\nu(\mathrm{d}t, \mathrm{d}x)$ of (\ref{eq:levy-khintchine}), 
we can obtain different NRMs. We use $\mbox{NRM}(\eta, M, H(\cdot))$ to denote the normalized random measure, 
where $M$ is the total mass, which usually needs to be sampled in the 
model, and $H(\cdot)$ is the base probability measure, 
$\eta$ is the set of other hyper-parameters
to the measure on the jumps,
depending on the specific NRMs. In this paper, 
we are interested in a class of NRMs called {\em normalized generalized Gamma processes}:

\begin{definition}[Normalized Generalized Gamma Processes]
 Generalized Gamma processes are random measures proposed by Brix~\cite{Brix:AAP99} for constructing
 shot noise Cox processes. They have the L\'{e}vy measures as
  \begin{equation}
   \nu(\mathrm{d}t, \mathrm{d}x) = \frac{e^{-bt}}{t^{1 + a}}H(\mathrm{d}x), b > 0, 0 < a < 1.
  \end{equation}
  By normalizing the generalized Gamma process as in (\ref{eq:NRM}), we obtain the normalized 
  generalized Gamma process (NGG).
\end{definition}
For ease of representation and sampling, we convert the NGG into a different
form using the following lemma.
\begin{lemma}
\label{lem:scale}
Let a normalised random measure be defined using  L\'{e}vy density
$\nu(\mathrm{d}x,\mathrm{d}t)$.  Then
scaling $t$ by $\lambda>0$ yields an equivalent NRM
up to a factor.
That is, 
the normalised measure obtained using $\nu(\mathrm{d}x,\mathrm{d}t/\lambda)$
is equivalent to the normalised measure obtained 
using $\lambda\,\nu(\mathrm{d}x,\mathrm{d}t)$.
\end{lemma}
By this lemma, without loss of generality,
we can instead represent the NGG by eliminating the
parameter $b$ above.
\begin{definition}[Normalized Generalized Gamma:]
The NGG with shape parameter $a$,
total mass (or concentration) parameter $M$ 
and base distribution $H(\cdot)$,
denoted $\mbox{NGG}(a,M,H(\cdot))$,
has L\'{e}vy density
$M\,\rho_{a}(\mathrm{d}t) H(\mathrm{d}x)$
where
\[
\rho_{a}(t) ~=~
\frac{a}{\Gamma(1-a)} \frac{e^{-t}}{t^{1+a}} ~.
\]
\end{definition}

Note that similar to the two parameter Poisson-Dirichlet process~\cite{PitmanYor97}, the normalized 
generalized Gamma process with $a\neq 0$ can also produce power-law phenomenon, 
making it different from the Dirichlet process and suitable to model real data. 

\begin{proposition}[\cite{LijoiMP:JRSS07}]\label{prop:powerlaw}
  Let $K_n$ be the number of components induced by the NGG with parameter $a$ 
and mass $M$
  or the Dirichlet process with total mass $M$. Then for the NGG,
  $K_n/n^a\rightarrow S_{a,M}$ almost surely, where $S_{a,M}$ is a strictly positive
  random variable parameterized by $a$ and $M$. For the DP, $K_n/\log(n)\rightarrow M$.
\end{proposition}

Figure~\ref{fig:powerlaw} demonstrates the power law phenomena in the NGG compared to the 
Dirichlet process (DP).
We sample it using the generalized Blackwell-MacQueen sampling
scheme~\cite{JamesLP:SJS09}.
Each data to be sampled can choose an existing cluster or create a new cluster,
resulting in $K$ clusters with $N$ data points in total. 

\begin{figure}
  \centering
  \begin{minipage}{0.49\linewidth}
    \includegraphics[width=0.9\linewidth]{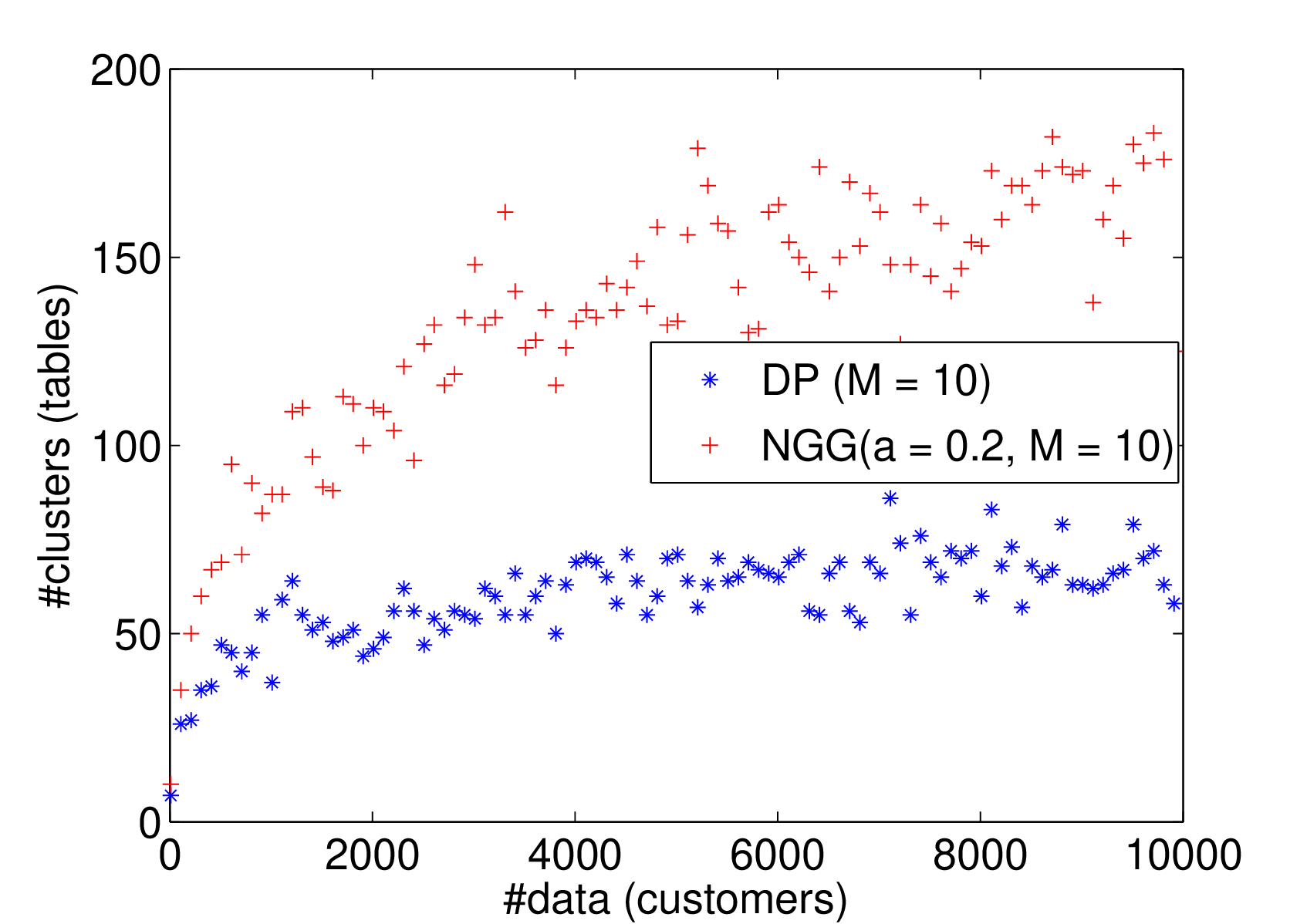}
  \end{minipage}
  \begin{minipage}{0.49\linewidth}
    \vspace{0.5cm}\includegraphics[width=0.99\linewidth]{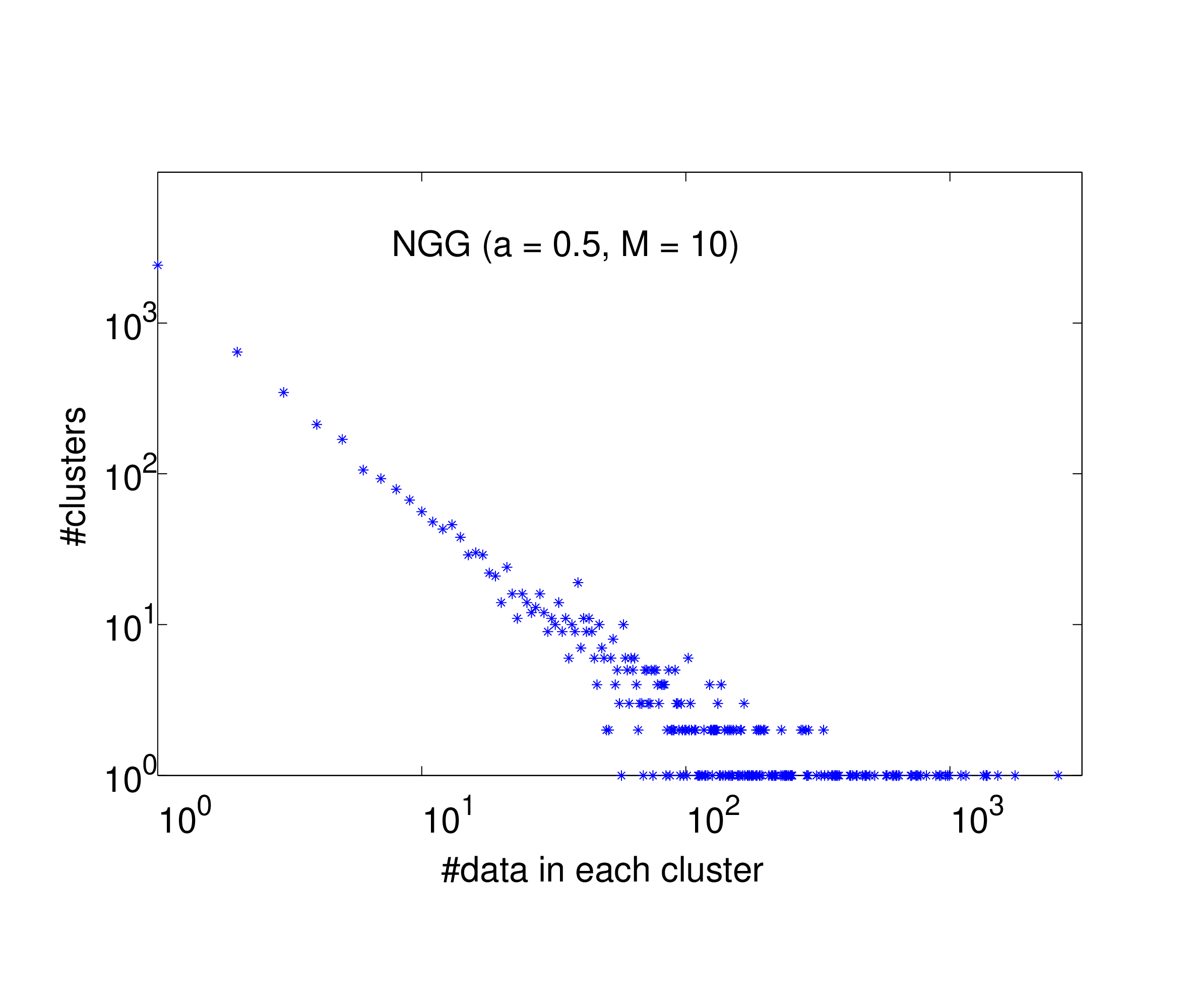}
  \end{minipage}
  \caption{Power-law phenomena in NGG. The first plot shows the \#data versus \#clusters compared with DP,
    the second plot shows the size $s$ of each cluster versus total number of clusters with size $s$.
} \label{fig:powerlaw}
\end{figure}

Many familiar stochastic processes are special/limiting cases of normalized generalized Gamma processes, 
{\it e.g.}, {\em Dirichlet processes} arise when $a\rightarrow0$. 
{\em Normalized inverse-Gaussian processes} (N-IG) arise when $a = \frac{1}{2}$ and $b=\frac{1}{2}$. 
If $b \rightarrow 0$, we get the {\em $\sigma$-stable process}, and if $a \rightarrow 0$ and $b$ depends on $x$, 
we get the {\em extended Gamma process}.

\begin{remark}
For the NGG, key formula used subsequently are as follows:
\begin{eqnarray*}
\psi_{a}(v) &=&  M \left((1+v)^a - 1\right)\\
\int_L^\infty\rho_{a}(\mathrm{d}t) &=& 
 |Q(-a,L)|\\
\int_L^\infty e^{-vx}\rho_{a}(\mathrm{d}t) &=& 
(1+v)^a|Q(-a,L(1+v))|  \\
\int_0^L \left( 1- e^{-vt} \right) \rho_{a}(t) \mathrm{d}t
 &=& \left((1+v)^a - 1\right)
 + (1+v)^a|Q(-a,L(1+v))| - |Q(-a,L)|
\end{eqnarray*}
where $Q(x,y)=\Gamma(x,y)/\Gamma(x)$ is the regularized
upper incomplete Gamma function.
Some mathematical libraries provide it for a negative first
argument, or it can be evaluated using
\[
Q(-a,z) ~=~ Q(1-a,z) - \frac{1}{\Gamma(1-a)}z^{-a}e^{-z},
\]
using an upper incomplete Gamma function defined
only for positive arguments.
\end{remark}

Finally, because probabilities for a NRM necessarily
have the divisor $\tilde\mu({\cal X}) ~=~\sum_{k=1}^{\infty}J_k$,
and thus likelihoods of the NRM should involve powers of $\tilde\mu({\cal X})$,
a trick is widely used to eliminate these terms.
\begin{definition}[Latent relative mass:]
Consider the case where $N$ data are observed.
By introducing the auxiliary variable,
called {\it latent relative mass},
$U_N={\Gamma_N}/{\tilde{\mu}({\cal X})}$
where $\Gamma_N \sim \Gamma(1,N)$,
then it follows that 
\[
\frac{1}{\tilde{\mu}({\cal X})^N} p(\Gamma_N)\mathrm{d}\Gamma_N
~=~ \exp\left\{-U_N \sum_{k=1}^{\infty}J_k\right\}\mathrm{d}U_N
\]
\end{definition}
Thus the $N$-th power of the normaliser can be replaced by
an exponential term in the jumps which factorizes,
at the expense of introducing the new latent variable $v$.
To the best of our knowledge, the idea of this latent variable
originals from~\cite{James:AST05} and is future explicitly studied
in~\cite{JamesLP:SJS06,JamesLP:SJS09,GriffinW:JCGS11}, {\it etc.}.

\subsection{Slice sampling normalized random measure mixtures}\label{sec:SNRM}

Slice sampling an NRM has been discussed in several papers, 
here we follow the method in~\cite{GriffinW:JCGS11}, to briefly introduce the ideas 
behind it. It deals with the normalized random measure mixture of the type
\begin{eqnarray}\label{eq:IMM}
 \hspace{-1cm}\mu = \sum_{k = 1}^\infty\omega_k\delta_{\theta_k},\nonumber\\
  \theta_{s_i}\sim\mu, \hspace{1cm}x_{i}\sim g_0(\cdot|\theta_{s_i})
\end{eqnarray}
where $\omega_k = J_k / \sum_{l = 1}^\infty J_l$, $J_1, J_2, \cdots$ are the jumps
of the corresponding CRM defined in (\ref{eq:RM}), $\theta_k$'s are 
the components of the mixture model drawn {\it i.i.d.}~from a parameter space $H(\cdot)$,
$s_i$ denotes the component that $x_i$ belongs to,
and $g_0(\cdot|\theta_k)$ is the density function to generate data on component $k$. 
Given the observations $\vec{x}$, 
we introduce a slice latent variable $u_i$ for each $x_i$ so that we only consider 
those components whose jump sizes $J_k$'s are larger than the corresponding $u_i$'s. 
Furthermore, the auxiliary variable $U_N$ (latent relative mass) is introduced to decouple each individual 
jump $J_k$ and their infinite sum of 
the jumps $\sum_{l = 1}^\infty J_l$ appeared in the denominators of $\omega_k$'s.
For clarification, we list the notation and its description in Table~\ref{tab:notations}. 
Based on~\cite{GriffinW:JCGS11}, we have the following posterior Lemma.

\begin{lemma}\label{lemma:sliceimf}
 The posterior of the infinite mixture model (\ref{eq:IMM}) with the above auxiliary
 variables is proportional to
  \begin{eqnarray}\label{eq:svmarglikelihood}
  &&\hspace{-1cm}p(\vec{\theta}, J_1, \cdots, J_K, K, \vec{u}, L, \vec{s}, U_N | \vec{x}, \mbox{NRM}(\eta,M,H(\cdot)) ) \propto \nonumber\\ 
    &&\hspace{-1cm}\exp\left\{-U_N\sum_{k = 1}^KJ_k\right\}
    \exp\left\{-M\int_0^L\left(1 - e^{-U_Nt}\right)\rho_{\eta}(t)\mathrm{d}t\right\} \nonumber\\
    &&\hspace{-1cm}U_N^{N - 1}p(J_1, \cdots, J_K)\prod_{k = 1}^Kh(\theta_k)\prod_{i = 1}^N 1(J_{s_i} > u_i)g_0(x_i | \theta_{s_i}),
  \end{eqnarray}
  where $1(a)$ is a indicator function returning 1 if $a$ is true and 0 otherwise, 
  $h(\cdot)$ is the density of $H(\cdot)$, $J^* = \sum_{k = K + 1}^\infty J_k$, $L = \min\{\vec{u}\}$,
  and $p(J_1, \cdots, J_K) = \prod_{k = 1}^K\frac{\rho_\eta(J_k)}{\int_L^\infty\rho_\eta(t)\mathrm{d}t}$ is the jump (large than $L$) 
  distribution derived from the underlying Poisson process (actually, $\vec{J}$ follows a compound Poisson process, meaning that
  it has $K\sim\text{Poisson}(M\int_L^\infty\rho_{\eta}(\mathrm{d}t))$ jumps, while each jump has density 
  $\frac{\rho_\eta(J_k)}{\int_L^\infty\rho_\eta(s)\mathrm{d}t}$, here $\text{Poisson}(x)$ means Poisson distribution with mean $x$).
\end{lemma}
The expressions for the NGG needed to work with this lemma
were given in the remark at the end of Section~\ref{sec:NRM}.  
Thus the integral term in Equation~(\ref{eq:svmarglikelihood}) 
can be turned into an expression involving incomplete Gamma functions.


{\small
\begin{table}[t]
  \caption{List of notation.}\label{tab:notations}
  \vskip 0.15in
  \begin{center}
    \begin{small}
      \begin{sc}
	\begin{tabular}{p{0.15\linewidth}|p{0.75\linewidth}}
	  \hline
	  Notation & Description \\
	  \hline
	  $K$ & \#components with jump sizes larger than a threshold $L$ \\
	  \hline
	  $\theta_k, k = 1, \cdots, K$ & Components in the mixture model \\
	  \hline
	  $M$ &   Total mass of the random measure \\
	  \hline
	  $J_k, k = 1, \cdots, K$ & Jump sizes of the random measure with all $J_k > L$ \\
	  \hline
	  $J^*$ & Sum of the remaining jump sizes, $J^{*} = \sum_{k = K + 1}^\infty J_k, J_k < L$ \\
	  \hline
	  $y_i, i = 1, \cdots, N$ & Observed data \\
	  \hline
	  $n_k, k = 1, \cdots, K$ & \#data attached to each component \\
	  \hline
	  $N$ & total number of data points \\
	  \hline
	  $s_i, i = 1, \cdots, N$ & Variables indicating which component $y_i$ belongs to \\
	  \hline
	  $u_i, i = 1, \cdots, N$ & slice variable uniformly distributed in $(0, J_{s_i}]$ for $y_i$ \\
	  \hline
	  $L$ & $L = \min\{\vec{u}\}$ \\
	  \hline
	  $U_N$ & An auxiliary variable introduced to make the sampling feasible \\
	  \hline
	  $g_0(\cdot|\theta_k)$ & Density function to generate data on component $\theta_k$ \\
	  \hline
	  $h(\theta_k)$ & Density of $H(\theta_k)$ \\
	  \hline
	  $p_M(M)$ & Prior for $M$ \\
	  \hline
	  $\nu(\mathrm{d}t, \mathrm{d}x)$ & L\'{e}vy measure of the random measure with 
	    decomposition $\nu(\mathrm{d}t, \mathrm{d}x) = \rho_{\eta}(\mathrm{d}t)H(\mathrm{d}x)$
	    considered in this paper \\
	  \hline
	\end{tabular}
      \end{sc}
    \end{small}
  \end{center}
  \vskip -0.1in
\end{table}
}

\subsubsection{Sampling:}\label{sec:slicesamplingMix}
First, we denote the parameter set as $C = \left\{\vec{\theta}, J_1, 
\cdots, J_K, K, \vec{u}, L, \vec{s}, U_N, M\right\}$, then the sampling goes as
\begin{itemize}
 \item {\bf Sampling $\vec{s}$: } From (\ref{eq:svmarglikelihood}) we get
  \begin{equation}
    p(s_i = k|C\backslash\{s_i\})\propto 1(J_k > u_i)g_0(y_i|\theta_k)
  \end{equation}

 \item {\bf Sampling $U_N$: }Similarly 
  \begin{eqnarray}
    p(U_N|C\backslash\{U_N\})&\propto& U_N^{N-1}\exp\left\{-U_N\sum_{k = 1}^KJ_k\right\} \nonumber\\
    &&\exp\left\{-M\int_{0}^L\left[1 - 
	\exp\left\{-U_Nt\right\}\right]\rho_{\eta}(\mathrm{d}t)\right\},
  \end{eqnarray}
  which can be sampled using rejection sampling from a proposal distribution 
$Ga\left(n, \sum_{k = 1}^KJ_k\right)$,
  here $Ga(a, b)$ means a Gamma distribution with shape parameter $a$ and scale parameter $b$.

 \item {\bf Sampling $\theta$: }The posterior of $\theta_k$ with prior density $h(\theta_k)$ is 
  \begin{equation}\label{eq:sampletheta}
   p(\theta_k|C\backslash\{\theta_k\})\propto h(\theta)\prod_{i|s_i = k}g_0(y_i|\theta_k).
  \end{equation}

 \item {\bf Sampling $K, \{J_1, \cdots, J_K\}$: }Sampling for $J_k$ can be done
  separately for those associated with data points (fixed points) and for those that are not.
  Based on~\cite{JamesLP:SJS09}, when integrating out $\vec{u}$ in (\ref{eq:svmarglikelihood}), 
  the posterior of the jump $J_k$ with data attached ($n_k > 0$) is proportional to
  \begin{equation}\label{eq:jumpsizeleq0}
   J_k^{n_k}\exp\left\{-U_NJ_k\right\}\rho_{\eta}(J_k),
  \end{equation}
   While for those without data attached ($n_k = 0$), based on~\cite{GriffinW:JCGS11}, conditional on $U_N$,
   the number of these jumps follows a Poisson distribution with mean 
    $$M\int_L^\infty\exp\{-U_Nt\}\rho_{\eta}(\mathrm{d}t),$$
    while their lengths $t$ have densities proportional to 
    $$\exp\{-U_Nt\}\rho_{\eta}(\mathrm{d}t)1(t > L).$$

  \item {\bf Sampling $\vec{u}$: }$\vec{u}$ are uniformly
    distributed in the interval $(0, J_{s_i}]$ for each $i$. After sampling the $\vec{u}$,
    $L$ is set to $L = \min\{\vec{u}\}$.

 \item {\bf Sampling $M$:} 
  The posterior of $M$ with prior $p_M(M)$ is 
  \begin{eqnarray}\label{eq:sampleM}
    &&\hspace{-1cm}p(M|C\backslash\{M\})\propto p_M(M)M^K\exp\left\{-M\left[\int_L^\infty\rho_{\eta}(\mathrm{d}t) \right.\right. 
    + \left.\left. \int_{0}^L\left[1 - \exp\left\{-U_Nt\right\}\right]\rho_{\eta}(\mathrm{d}t)\right]\right\}.
  \end{eqnarray}
  $p_M(M)$ is usually taken to be Gamma distributed, so the posterior of $M$ 
  can be sampled conveniently.
\end{itemize}

\section{Operations}\label{sec:ops}

This section introduces the dependency operations used.
These are developed for Poisson processes, CRMs and
NRMs.

\subsection{Operations on Poisson processes}\label{sec:PoiPOP}

We review three operations that transform Poisson processes in order to construct dependent
completely random measures. For details, refer to~\cite{Kindman:PoiP10,LinGF:NIPS10}.

\begin{definition}[Superposition of Poisson processes]
 Given a set of Poisson processes $\Pi_1, \Pi_2, \cdots, \Pi_n$, the superposition 
  of these Poisson processes is defined as the union of the points in these 
  Poisson processes:
  \begin{equation}
   \Pi := \bigcup_{i=1}^n\Pi_i.
  \end{equation}
\end{definition}

\begin{lemma}[Superposition Theorem]
 Let $\Pi_1, \cdots, \Pi_n$ be $n$ independent Poisson processes on $\mathbb{S}$ 
  with $\Pi_k\sim\text{PoissonP}(\nu_k)$, then the superposition of these 
  $n$ Poisson processes is still a Poisson process with 
  $\Pi\sim\text{PoissonP}(\sum_i\nu_i)$.
\end{lemma}

\begin{definition}[Subsampling of Poisson processes]
 Subsampling of a Poisson process with sampling rate $q(\theta)$ is defined 
 to be selecting the points of the Poisson process via independent Bernoulli 
  trials with acceptance rate $q(\theta)$.
\end{definition}

\begin{lemma}[Subsampling Theorem]
 Let $\Pi\sim\text{PoissonP}(\nu)$ be a Poisson process on the space $\mathbb{S}$ and 
  $q: \mathbb{S}\rightarrow [0, 1]$ be a measurable function. If we independently draw 
  $z_\theta\in\{0, 1\}$ for each $\theta\in\Pi$ with $P(z_\theta = 1) = q(\theta)$, 
  and let $\Pi_k = \{\theta\in\Pi: z_\theta = k\}$ for $k = 0, 1$, then 
  $\Pi_0, \Pi_1$ are independent Poisson processes on $\mathbb{S}$ with 
  $S^{1 -q}(\Pi) := \Pi_0\sim\text{PoissonP}((1 - q)\nu)$ and $S^q(\Pi) := \Pi_1\sim\text{PoissonP}(q\nu)$.
\end{lemma}

\begin{definition}[Point transition of Poisson processes]
  Point transition of a Poisson process $\Pi$ on space $(\mathbb{S}, \mathcal{S})$, denoted as
  $T(\Pi)$, is defined as moving each point of the Poisson process independently 
  to other locations following a probabilistic transition kernel $T$, which 
  is defined to be a function $\mathcal{T}: \mathbb{S}\times\mathcal{S}\rightarrow [0, 1]$
  \footnote{In the following we will use $T(\cdot)$ to denote the point transition operation, while
  use $\mathcal{T}(\cdot, \cdot)$ to denote the corresponding transition kernel.} such 
  that for each $\theta\in \mathbb{S}$, $\mathcal{T}(\theta, \cdot)$ 
  is a probability measure on $E$ that describes the distribution of where the point 
  $\theta$ moves, and for each $A\in\mathcal{S}$, $\mathcal{T}(\cdot, A)$ is integrable.
  Thus, $T(\Pi) := \{\theta^\prime: \theta^\prime\sim \mathcal{T}(\theta, \cdot)|\theta\in\Pi\}$.
  With a little abuse of notation, we use $\mathcal{T}(\theta)$ to denote a sample from $\mathcal{T}(\theta, \cdot)$
  in this paper.
Thus $\mathcal{T}(\theta)$ is a stochastic function.
\end{definition}

\begin{lemma}[Transition Theorem]
 Let $\Pi\sim\text{PoissonP}(\nu)$ be a Poisson process on space $(\mathbb{S}, \mathcal{S})$, $\mathcal{T}$ 
 a probability transition kernel, then
  \begin{equation}
   T(\Pi) \sim\text{PoissonP}(\mathcal{T}\nu).
  \end{equation}
  where $\mathcal{T}\nu$ can be considered as a transformation of measures over $\mathbb{S}$ defined as 
  $(\mathcal{T}\nu)(A):=\int_{\mathbb{S}}\mathcal{T}(\theta, A)\nu(\mathrm{d}\theta)$ for $A\in\mathbb{S}$.
\end{lemma}

\subsection{Operations on random measures}\label{sec:URMPOP}

\subsubsection{Operations on CRMs}
The dependency operations defined on Poisson processes in Section~\ref{sec:PoiPOP} can be naturally generalized
to the completely random measures given the construction in (\ref{eq:RM}). Formally, we have

\begin{definition}[Superposition of CRMs]
 Given $n$ independent CRMs $\tilde{\mu}_1, \cdots, \tilde{\mu}_n$ on $\mathbb{X}$, 
  the superposition ($\tilde{\oplus}$) is defined as:
  \begin{equation}
   \tilde{\mu}_1\tilde{\oplus} \tilde{\mu}_2\tilde{\oplus}\cdots\tilde{\oplus} \tilde{\mu}_n := \mu_1 + \mu_2 + \cdots + \mu_n ~. \nonumber
  \end{equation}
\end{definition}

\begin{definition}[Subsampling of CRMs]
 Given a CRM $\tilde{\mu} = \sum_{k = 1}^\infty J_k\delta_{\theta_k}$ on $\mathbb{X}$, and a measurable 
  function $q:\mathbb{X}\rightarrow[0, 1]$. If we independently draw $z(\theta)\in\{0, 1\}$ for each $\theta\in\mathbb{X}$
  with $p(z(\theta) = 1) = q(\theta)$,
  the subsampling of $\tilde{\mu}$, is defined as
  \begin{equation}
   \tilde{S}^q(\tilde{\mu}) := \sum_{k}z(\theta_k)J_k\delta_{\theta_k},
  \end{equation}
\end{definition}

\begin{definition}[Point transition of CRMs]
 Given a CRM $\tilde{\mu} = \sum_{k = 1}^\infty J_k\delta_{\theta_k}$ on $\mathbb{X}$, the point transition 
  of $\tilde{\mu}$, is to draw atoms $\theta_k^\prime$
  from a transformed base measure to yield a new random measure as
   $$\tilde{T}(\tilde{\mu}) := \sum_{k = 1}^\infty J_k\delta_{\theta_k^\prime}.$$
\end{definition}

\subsubsection{Operations on NRMs}

The operations on NRMs can be naturally generalized from those on CRMs:
\begin{definition}[Superposition of NRMs]
 Given $n$ independent NRMs $\mu_1, \cdots, \mu_n$ on $\mathbb{X}$, 
  the superposition  ($\oplus$) is:
  \begin{equation}
   \mu_1\oplus \mu_2\oplus\cdots\oplus \mu_n := c_1\mu_1 + c_2\mu_2 + \cdots + c_n\mu_n ~. \nonumber
  \end{equation}
where the weights $c_m=\frac{\tilde{\mu}_m(\mathbb{X})}{\sum_j\tilde{\mu}_j(\mathbb{X})}$ and 
$\tilde{\mu}_m$ is the unnormalized random measures corresponding to $\mu_m$.
\end{definition}

\begin{definition}[Subsampling of NRMs]
 Given a NRM $\mu = \sum_{k = 1}^\infty r_k\delta_{\theta_k}$ on $\mathbb{X}$, and a measurable 
  function $q: \mathbb{X}\rightarrow[0, 1]$. If we independently draw 
  $z(\theta)\in\{0, 1\}$ for each $\theta\in\mathbb{X}$
  with $p(z(\theta) = 1) = q(\theta)$, the subsampling of $\mu$, is defined as
  \begin{equation} \label{eq:subsampling}
   S^q(\mu) := \sum_{k:z(\theta_k) = 1}\frac{r_k}{\sum_{j}z(\theta_j)r_j}\delta_{\theta_k},
  \end{equation}
\end{definition}

\begin{definition}[Point transition of NRMs]
 Given a NRM $\mu = \sum_{k = 1}^\infty r_k\delta_{\theta_k}$ on $\mathbb{X}$, the point transition 
  of $\mu$, is to draw atoms $\theta_k^\prime$
  from a transformed base measure to yield a new NRM as
   $$T(\mu) := \sum_{k = 1}^\infty r_k\delta_{\theta_k^\prime}~.$$
\end{definition}

The definitions are constructed so the following simple lemma holds.
\begin{lemma}\label{lem:OECR}
Superposition, subsampling or point transition of NRMs
is equivalent to superposition, subsampling or point transition of 
their underlying CRMs.
\end{lemma}
Thus one does not need to distinguish between whether these operations
are on CRMs or NRMs.

\section{Posteriors for the NGG}\label{sec:postngg}

This section develops posteriors for the single NGG,
for a standard version 
$p\left(\vec{X}|\mbox{NGG}(a,M,H(\cdot)\right)$
and a version conditioned on the latent relative mass $U_N$,
$p\left(\vec{X}|U_N, \mbox{NGG}(a,M,H(\cdot)\right)$.
The second version is done because, as shown, the first version
requires computing a complex recursive function.

\subsection{Simple Posterior}\label{sec:post}

James {\it et al.}\ \cite{JamesLP:SJS09} develop posterior
analysis as follows.  
This theorem simplifies their results and specialises them to the NGG.
\begin{theorem}[Posterior Analysis for the NGG]\label{thm:postngg}
Consider the $\mbox{NGG}(a,M,H(\cdot))$.
For a data vector $\vec{X}$ of length $N$ there are $K$ distinct
values $X^*_1,...,X^*_K$ with counts
$n_1,...,n_K$ respectively.
The posterior marginal is given by
\begin{equation}
p\left(\vec{X}|\mbox{NGG}(a,M,H(\cdot)\right) ~=~
\label{eq:postngg}
 \frac{e^{M}a^{K-1}T^{N,K}_{a,M} }{\Gamma(N)}
\prod_{k=1}^K (1-a)_{n_k-1} h(X^*_k) ~.
\end{equation}
where 
\begin{equation}
\label{eq:defT}
T^{N,K}_{a,M} ~=~\int_{M}^\infty
\left(1-\left(\frac{M}{t}\right)^{1/a}\right)^{N-1}t^{K-1}e^{-t}\mathrm{d}t
~.
\end{equation}
Moreover, the predictive posterior is given by:
\[
p(X_{N+1}\in \mathrm{d}x|\vec{X},\mbox{NGG}\left(a,M,H(\cdot)\right) ~=~
\omega_0 H(\mathrm{d}x) + \sum_{k=1}^K
\omega_k \delta_{X^*_k}(\mathrm{d}x)
\]
where the weights sum to 1 ($\sum_{k=0}^K \omega_k=1$)
are derived as
\begin{eqnarray}
\nonumber
\omega_0 &\propto & a \frac{T^{N+1,K+1}_{a,M}}{T^{N+1,K}_{a,M}} \\
\omega_k &\propto & (n_k-a) 
\label{eq:omegadef}
\end{eqnarray}
\end{theorem}
Note that an alternative definition of $T^{N,K}_{a,M}$ is
\[
 T^{N,K}_{a,M}~=~a \frac{M^K}{e^M}
   \int_{\mathbb{R}^+} \frac{u^{N-1}}{(1+u)^{N-Ka}} e^{M-M(1+u)^a}
   \mathrm{d}u~,
\]
and various scaled versions of this integral are presented
in the literature.
Introducing a $\Gamma(b/a,1)$ prior on $M$ and then marginalising
out $M$ makes the term in $e^{M-M(1+u)^a}$ disappear
since the integral over $M$ can be carried inside the
integral over $u$.
\begin{corollary}\label{cor:pdd}
Let $\vec\mu\sim \mbox{NGG}\left(a,M,H(\cdot )\right)$ 
and suppose $M \sim\Gamma(b/a,1)$
then it follows that 
$\vec\mu \sim \mbox{PDP}(a,b,H(\cdot ))$
\end{corollary}

For computation, the issue here will be computing
the terms $T^{N,K}_{a,M}$.
Therefore we present some results for this.
\begin{lemma}[Evaluating $T^{N,K}_{a,M}$:]\label{lem:eval}
Have $T^{N,K}_{a,M}$ defined as in Theorem~\ref{thm:postngg}.
Then the following formula hold:
\begin{eqnarray}
\label{eq:le1}
T^{N,K}_{a,M} &\le& \Gamma(K,M) ~,\\
\label{eq:le5}
T^{N,K}_{a,M} &=&
\sum_{n=0}^{N-1} {N-1 \choose n} \left(-M^{1/a}\right)^{n} 
        \frac{\Gamma\left(K-\frac{n}{a}\right)}{\Gamma\left(1-\frac{n}{a}\right)}  \Gamma\left(1-\frac{n}{a},M\right)
~~
\forall N,K\in \mathbb{N}^+~,
\\ 
\label{eq:le6}
T^{N,K+2}_{a,M} &=& K\,T^{N+1,K}_{a,M} + \frac{N-1}{a}
  \left(T^{N,K}_{a,M} -T^{N,K+1}_{a,M} \right)~,~~~~~~
\forall N>2,K\in \mathbb{N}^+
\end{eqnarray}
where $\Gamma(x,y)$ is the upper incomplete gamma function,
defined for $y>0$ and $x\ne 0, -1, -2, ...$.
Moreover, for Equation~(\ref{eq:le5}), 
$k a$ cannot be integral for $k=1,...,K-1$.

Another recursion is needed when 
$a=1/R$ for some $R\in \mathbb{N}^+, R>1$.
Then 
\begin{equation}
\label{eq:le7}
T^{N+1,K}_{a,M} ~=~ T^{N,K}_{a,M} - M^{1/a} T^{N,K-1/a}_{a,M} ~,~~~~~~~~\forall K>R, N,K\in \mathbb{N}^+~,
\end{equation}
\end{lemma}
It can be seen there are two different situations.
When $a=1/R$ for some $R\in \mathbb{N}^+, R>1$, then one can recurse down
on $N$.  But otherwise, one recurses down on $K$.
Moreover, $T^{N,K}_{a,M}$ is a strictly decreasing function of
$N$ and $M$, but an increasing function of $K$ and $a$.
For computation, Equation~(\ref{eq:le5}) can be used to compute
$T_{a,M}^{N,1}$ and $T_{a,M}^{N,2}$ 
in terms of
$\Gamma\left(1-\frac{N}{a}\right)$.
This equation may not be
usable for $K>2$ and may be unstable.
Thereafter, for $K>2$ in $T_{a,M}^{N,K}$
the recursion of Equation~(\ref{eq:le6}) can be applied.

\begin{remark}
The Poisson-Dirichlet Process and Dirichlet Process
are well known for their ease
of use in a hierarchical context
\cite{TehJorBea2006,ChenDB:ECML11,Bun12}.  The NGG has the
same general form.
\end{remark}
The major issue with this posterior theory is that one
needs to precompute the terms 
$T^{N,K}_{a,M}$.
While the Poisson-Dirichlet Process has a 
similar style, it has a generalised Stirling number dependent only on the
discount $a$ \cite{Bun12}.  
The difference is that for the PDP
we can tabulate these terms for a given discount parameter
$a$ and still vary the concentration parameter
($b$ above, but corresponding to $M$) easily.
For the NGG, any tables of $T^{N,K}_{a,M}$ would need to
be recomputed with every change in mass parameter $M$.
This might represent a significant computational burden.

\subsection{Conditional Posterior}\label{sec:cpost}
James {\it et al.}\ \cite{JamesLP:SJS09} also develop 
conditional posterior
analysis as follows.  
This theorem simplifies their results and specialises them to the NGG.
\begin{theorem}[Conditional Posterior Analysis for the NGG]\label{thm:cpostngg}
Consider the NGG$_{a,M}$ and the situation of
Theorem~\ref{thm:postngg}.
The conditional posterior marginal, 
conditioned on the auxiliary variable $U_N$, is given by
\begin{equation}
p\left(\vec{X}|U_N=u,\mbox{NGG}\left(a,M,H(\cdot)\right),N\right)
~=~
\label{eq:cpostngg}
\frac{\left(Ma\left(1+u\right)^a\right)^K }
{\sum_{k=1}^N S^N_{k,a}  \left(Ma\left(1+u\right)^a\right)^k   }
\prod_{k=1}^K (1-a)_{n_k-1} h(X^*_k) ~.
\end{equation}
Moreover, the predictive posterior is given by:
\[
p\left(X_{N+1}\in \mathrm{d}x|\vec{X},U_N=u,\mbox{NGG}\left(a,M,
         H(\cdot)\right),N\right) ~=~
\omega_0 H(\mathrm{d}x) + \sum_{k=1}^K
\omega_k \delta_{X^*_k}(\mathrm{d}x)
\]
where the weights sum to 1 ($\sum_{k=0}^K \omega_k=1$)
are derived as
\begin{eqnarray}
\nonumber
\omega_0 &\propto & 
Ma\left(1+u\right)^a \\
\label{eq:comegadef}
\omega_k &\propto &  n_k-a ~.
\end{eqnarray}
The posterior for $U_N$ is given by:
\begin{equation}
\label{eq:cpostngg:u}
p\left(U_N=u|\vec{X},\mbox{NGG}\left(a,M,H(\cdot)\right),N\right) 
~=~
\frac{a M^K}{T^{N,K}_{a,M}}
\frac{u^{N-1}}{\left(1+u\right)^{N-Ka}} 
e^{-M \left(1+u\right)^a} ~.
\end{equation}
\end{theorem}
A posterior distribution is also presented by James {\it et al.}\ as
their major result of Theorem~1 \cite{JamesLP:SJS09}.
We adapt it here to the NGG.
\begin{theorem}
In the context of Theorem~\ref{thm:cpostngg} the conditional
posterior of the normalised random measure
$\tilde\mu$ given data $\vec{X}$ of length $N$ and latent relative mass 
$U_N=u$
is given by
\[
\vec\mu ~=~ \frac{T}{T+J_+}\vec\mu' 
         + \frac{J^+}{T+J_+} \sum_{k=1}^K p_k \delta_{X^*_k}
\]
where 
\begin{eqnarray*}
\vec\mu' &\sim& \mbox{NGG}\left(a,\frac{M}{1+u},H(\cdot )\right)~,\\
T &\sim &f_T(t)~~~\mbox{where L\'{e}vy measure of $f_T(t)$}
=\frac{M a}{\Gamma(1-a)} s^{-a-1}e^{-(1+u)s}~,\\
J^+ &\sim& \Gamma(N-Ka,1+u)~,\\
\vec{p}&\sim& \mbox{Dirichlet}_K\left(\vec{n}-a\right)~.
\end{eqnarray*}
Here, $\vec\mu'$, $J_+$ and $\vec{p}$ are jointly independent
and $T$, $J_+$ and $\vec{p}$ are jointly independent.
\end{theorem}
Note in particular the densities given for $\vec\mu'$ and $T$ are
not independent from each other.  While an explicit density is not given for
$T$, its expected value is easily computed via the Laplace
transform as $Ma(1+u)^{a-1}$.

Griffin {\it et al.}\ \cite{GriffinW:JCGS11}
present an alternative technique for obtaining the conditional
posterior.    The following is adapted from their main sampler 
after integrating out the slice variables.
\begin{theorem}[Sampling Posterior]\label{thm:grif}
Consider a bound $0<L<\infty$ which is
sufficiently small so that it is less than the jumps $J_k$
associated with all the observed data.
For an NRM given by $\mbox{NRM}\left(\eta,M, H(\cdot )\right)$,
the number of jumps $K_L$ with value
$J_k>L$ is a random variable as well as their values $J_1,...,J_{K_L}$.
The resultant posterior is as follows:
\begin{eqnarray*}
\lefteqn{p\left(\vec{X},U_N=u\,|\,K_L,J_1,...,J_{K_L},N,
      \mbox{NRM}\left(\eta,M, H(\cdot )\right)\right)}
&&\\
&=& u^{N-1}e^{-M\int_0^L \left( 1- e^{-us} \right) \rho_{\eta}(s) \mathrm{d}s }
\prod_{k=1}^{K_L} e^{-(1+u)J_k} J_k^{n_k} h(X^*_k)^{1_{n_k>0}}
\end{eqnarray*}
where $X^*_k$ are the unique data values (from $\vec{X}$)
and $n_k$ are the count of data from $\vec{X}$ having the value
$X^*_k$.
\end{theorem}
The expressions for the NGG needed to work with this lemma
were given in the remark at the end of Section~\ref{sec:NRM}.  
We further simplify this by marginalising out jumps $J_k$ 
 and then taking the limit as $L\rightarrow 0$.
Note we have  renumbered indexes so that
$n_k>0$ for all $k=1,...,K$ where $K\le K_L$.
This matches the conditionals of Theorem~\ref{thm:cpostngg}
so is seen to be correct.
\begin{corollary}[Reduced Sampling Posterior]\label{cor:grif}
In the context of Theorem~\ref{thm:grif},
assume there are $K$ jumps with attached data such that $n_k>0$.
The resultant posterior is as follows:
\begin{eqnarray}
\nonumber
\lefteqn{p\left(\vec{X},U_N=u,K\,|\,N,\mbox{NGG}\left(a,M,H(\cdot )\right)\right)}
&&\\
\label{eq:corgrif1}
&=& \frac{u^{N-1}}{(1+u)^{N-Ka}} 
\left(Ma \right)^{K}
e^{M - M(1+u)^a}
\prod_{k=1}^{K} (1-a)_{n_k-1} h(X^*_k) ~.
\end{eqnarray}
Moreover,
\begin{eqnarray}
\nonumber
\lefteqn{p\left(\vec{X},U_N=u,K,J_1,...,J_K\,|\,N,\mbox{NGG}\left(a,M,H(\cdot )\right)\right)}
&&\\
\label{eq:corgrif2}
&=& u^{N-1} \left( \frac{Ma}{\Gamma(1-a)}\right)^{K}
e^{M - M(1+u)^a}
\prod_{k=1}^{K} J_k^{n_k-a-1}e^{-(1+u) J_k} h(X^*_k) ~. 
\end{eqnarray}
\end{corollary}
\begin{remark}
With the use of the latent relative mass $U_N$,
the NGG lends itself to hierarchical reasoning without a need
to compute the recursive series $T^{N,K}_{a,M}$.
This can be done with either the jumps integrated out,
or the jumps retained.
\end{remark}

\section{Dependencies and Properties of Operations}\label{sec:depprop}

This section presents a number of results to do with the
operations applied to the NRMs.
First dependencies such as covariances are presented.
Then some further properties are developed for when the 
operations are used in a network.

\subsection{Dependencies between NRMs via Operations}\label{sec:dependentNRM}
Properties of the NRMs here are given in terms of the
Laplace exponent and its derivatives.
In the Dirichlet process case, we have 
$\psi(v) = M\log(1 + v)$,
while in the normalized generalized Gamma process case, we have
$\psi_{a}(v) = M\left((1 + v)^a - 1\right)$.
Because the dependencies involve the total masses significantly,
we use a modified version of the Laplace exponent in all these
results.
Define $\tilde\psi_\eta(v)=\frac{1}{M}\psi_\eta(v)$,
which has the mass removed.

Different from the Dirichlet process, the total masses
 $M$ are no longer independent
from their normalized jumps in general normalized random measures. However, we can
still derive the correlations between different NRMs. The following Theorems summarize
these results.



\begin{lemma}[Mean and Variance of an NRM]\label{lemma:covNRM}
 Given a normalized random measure $\mu$ on $\mathbb{X}$ with the underlying L\'{e}vy measure 
  $\nu(\mathrm{d}t, \mathrm{d}x) = M\rho_{\eta}(\mathrm{d}t)P(\mathrm{d}x)$, for
  $\forall B\in\mathcal{B}(\mathbb{X})$. The mean of this NRM is given by
  \begin{equation}
   \mathbb{E}[\mu(B)] = P(B)~.
  \end{equation}
 The variance of this NRM is given by
  \begin{eqnarray} \label{eq:varNRM}
    \text{Var}(\mu(B)) &=& P(B)(P(B) - 1)M \nonumber\\
    &&\hspace{-0.5cm}\int_0^\infty v\tilde\psi_{\eta}^{\prime\prime}(v)
	  \exp\left\{-M\tilde\psi_{\eta}(v)\right\}\mathrm{d}v~.
  \end{eqnarray}
\end{lemma}

\begin{remark}
 For DP, the corresponding variances are: 
$$\text{Var}_{DP}(\mu(B)) = \frac{P(B)(1 - P(B))}{M+1}.$$
  For NGG, it is $$\text{Var}_{NGG}(\mu(B)) = P(B)(1 - P(B))\frac{1 - a}{a}e^MM^{\frac{1}{a}}|\Gamma(-\frac{1}{a}, M)|.$$
For large $M$ the upper incomplete gamma function used here
has the property that $e^MM^{1+\frac{1}{a}}|\Gamma(-\frac{1}{a}, M)|\rightarrow 1$
and so we get for large $M$
\[
\text{Var}_{NGG}(\mu(B)) \rightarrow P(B)(1 - P(B))\frac{1 - a}{Ma} ~.
\]
\end{remark}

\begin{theorem}[Dependency via superposition]\label{theo:covSuper}
 Suppose $\mu_i, i = 1, \cdots, n$ are $n$ independent normalized random measures 
  on $\mathbb{X}$ with the underlying L\'{e}vy measures 
  $\nu_i(\mathrm{d}t, \mathrm{d}x) = M_i\rho_{\eta}(\mathrm{d}t)P(\mathrm{d}x)$, 
  let $\mu = \mu_1\oplus\cdots\oplus\mu_n$,
  $B\in\mathcal{B}(\mathbb{X})$, then the covariance between $\mu_k (k < n)$ and $\mu$ is
  \begin{eqnarray}\label{eq:superPCov}
   &&\text{Cov}\left(\mu_k(B), \mu(B)\right) = \nonumber\\
   &&\hspace{-1cm}P(B)M_k\int_0^\infty \gamma(M_k, P(B), v)
      \exp\left\{-(\sum_{j\neq k}M_j)\tilde\psi_\eta(v)\right\}\mathrm{d}v \nonumber\\ 
    &\hspace{-1cm}+&\hspace{-0.7cm} P(B)^2\left(1 - \frac{2\sum_{i\neq k}M_i}{\sum_jM_j}\right)~.
  \end{eqnarray}
  where 
  \begin{eqnarray}
    &&\gamma(M_k, P(B), v) = \\
    &&\hspace{-1cm}\int_0^v\left(P(B)M_k\tilde\psi_\eta^\prime(v_1)^2 - \tilde\psi_\eta^{\prime\prime}(v_1)\right)\exp\left\{-M_k\tilde\psi_\eta(v_1)\right\}\mathrm{d}v_1 \nonumber
  \end{eqnarray}
\end{theorem}

\begin{theorem}[Dependency via subsampling]\label{theo:covSub}
  Let $\tilde{\mu}$ be a completely random measure on $\mathbb{X}$ with L\'{e}vy measure 
  $\nu(\mathrm{d}t, \mathrm{d}x) = M\rho_{\eta}(\mathrm{d}t)P(\mathrm{d}x)$,
  $\mu = \frac{\tilde{\mu}}{\tilde{\mu}(\mathbb{X})}$.
  The covariance between $\mu$ and its subsampling version $S^q(\mu)$, denoted as $\mu^q$, with 
  sampling rate $q(\cdot)$ on $B\in\mathcal{B}(\mathbb{X})$ is
  \begin{eqnarray}\label{eq:covsubS}
   &&\text{Cov}\left(\mu^q(B), \mu(B)\right) = \nonumber\\
   &&\hspace{-1cm}P(B)M_q\int_0^\infty \gamma(M_q, P(B), v)
      \exp\left\{-(M - M_q)\tilde\psi_\eta(v)\right\}\mathrm{d}v \nonumber\\ 
    &+& P(B)^2\left(\frac{2M_q - M}{M}\right)~,
  \end{eqnarray}
  where $M_q := (q\tilde{\mu})(\mathbb{X}) = \int_\mathbb{X}q(x)\tilde{\mu}(x)\mathrm{d}x$.
\end{theorem}

\begin{theorem}[Dependency via point transition]\label{theo:covPoiTrS}
  Let $\tilde{\mu}$ be a random measure on $\mathbb{X}$ with L\'{e}vy measure 
  $\nu(\mathrm{d}t, \mathrm{d}x) = M\rho_{\eta}(\mathrm{d}t)P(\mathrm{d}x)$,
  $\mu = \frac{\tilde{\mu}}{\tilde{\mu}(\mathbb{X})}$. Let $B\in\mathcal{B}(\mathbb{X})$,
  $A = \mathcal{T}(B) := \{x: x\sim \mathcal{T}(y, \cdot), y\in B\}$ be the set of points obtained after the point transition on $B$,
  thus $P(A) = \int_BP(\mathcal{T}(x))\mathrm{d}x$. Suppose $A$ and
  $B$ are disjoint (which is usually the case when the transition operator $T$ is appropriately
  defined), the covariance between $\mu$ 
  and its point transition version $T(\mu)$ on $B\in\mathcal{B}(\mathbb{X})$ is
  \begin{eqnarray}
    &&\text{Cov}\left(\mu(B), (T\mu)(B)\right) = P(A)P(B) \\
    &&\hspace{-1cm}\left(M^2\int_0^\infty\int_0^{v_1} \tilde\psi_{\eta}^\prime(v_2)^2\exp\left\{-M\tilde\psi_\eta(v_2)\right\}\mathrm{d}v_2\mathrm{d}v_1 - 1\right) \nonumber
  \end{eqnarray}
\end{theorem}

\subsection{Properties of  the three dependency operations}

We first prove the following two 
Lemmas about superposition and subsampling of CRMs.

A straightforward extension of~Theorem~1 \cite{JamesLP:SJS09} leads to the 
following Lemma about the posterior of CRMs under superposition.

\begin{lemma}[Posterior of CRMs under superposition]\label{theo:PosSup}
 Let $\tilde{\mu}_1, \tilde{\mu}_2, \cdots, \tilde{\mu}_n$ be $n$ independent CRMs defined on space $\mathbb{X}$,
 with L\'{e}vy measures $\nu_i(\mathrm{d}t, \mathrm{d}x)$ 
for $i = 1, \cdots, n$.
  Let
  \begin{equation}\label{eq:superp}
    \tilde{\mu} = \oplus_{i = 1}^n\tilde{\mu}_i.
  \end{equation}
  Then given observed data $X = \{X_i\}$ (we use $X_k^{*}$ to denote the distinct values among $X$) and a latent relative mass $U_n$,
 the posterior of $\tilde{\mu}$ is given by
  (we use $x|(y)$ to denote the variable $x$ conditioned on $y$)
  \begin{equation}
    \tilde{\mu}|(U_n, X) = \tilde{\mu}|(U_n) + \sum_{k = 1}^KJ_k\delta_{X_k^{*}},
  \end{equation}
  where 
  \begin{enumerate}
   \item $\tilde{\mu}|(U_n)$ is a CRM with L\'{e}vy measure
      $$\nu(\mathrm{d}t, \mathrm{d}x) = e^{-ut}\left(\sum_{i = 1}^n\nu_i(\mathrm{d}t, \mathrm{d}x)\right),$$
    \item $X_k^{*}$ ($k = 1, \cdots, K$) are the fixed points of discontinuity and $J_k$'s are the
	corresponding jumps with densities proportional to $$t^{n_k}e^{-ut}\left(\sum_{i = 1}^n\nu_i(\mathrm{d}t, \mathrm{d}x)\right),$$
	where $n_k$ is the number of data attached at jump $J_k$.
    \item $\tilde{\mu}|(U_n)$ and $J_k$'s are independent.
  \end{enumerate}
\end{lemma}

By subsampling, we can prove the following formula of the L\'{e}vy measure under subsampling.
\begin{lemma}[L\'{e}vy measure under subsampling] \label{theo:Levysub}
 Let $\tilde\mu = \sum_{k = 1}^\infty J_k\delta_{X_k^{*}}$
 be a CRM  with L\'{e}vy measure 
 $\nu(\mathrm{d}t, \mathrm{d}x)$. Let $S^q(\tilde\mu)$ be its subsampling version with 
  acceptance rate $q(\cdot)$, 
then $S^q(\tilde\mu)$ has the L\'{e}vy measure of $q(\mathrm{d}x)\nu(\mathrm{d}t, \mathrm{d}x)$.
\end{lemma}

Now we give some properties about compositions of 
of the dependency operations which follow simply.
\begin{lemma}[Composition of dependency operators]\label{lem:comp}
Given CRMs $\tilde\mu$, $\tilde\mu'$ and $\tilde\mu''$, the following hold:
\begin{itemize}
\item
Two subsampling operations are commutative.
So with acceptance rates $q(\cdot)$ and $q'(\cdot)$,
then $S^{q'}(S^q(\tilde\mu))=S^q(S^{q'}(\tilde\mu))$.
Both are equal to $S^{q'q}(\tilde\mu))$.
\item
A constant subsampling operation commutes with a point
transition operation.
Thus  $S^{q}(T(\tilde\mu))=T(S^q(\tilde\mu))$
where the acceptance rate $q$ is indepenent of the data space.
\item
Subsampling and point
transition operations distribute over superposition.
Thus for acceptance rate $q(\cdot)$ and point transition
$T(\cdot)$,
\[
S^q(\tilde\mu \oplus \tilde\mu') ~=~ S^q(\tilde\mu) \oplus S^q(\tilde\mu') ~,
~~~~~~~~~~~T(\tilde\mu \oplus \tilde\mu') ~=~ T(\tilde\mu) \oplus T(\tilde\mu')~.
\]
\item
Superposition is commutative and associative.
Thus $\tilde\mu \oplus \tilde\mu' = \tilde\mu' \oplus \tilde\mu$ and
$(\tilde\mu \oplus \tilde\mu') \oplus \tilde\mu''
= \tilde\mu \oplus (\tilde\mu' \oplus \tilde\mu'')$.
\end{itemize}
\end{lemma}
Thus when subsampling operations are all constant,
a composition of subsampling, point
transition and superposition operations admits a normal form
where all the subsampling operations are applied first, then 
the transition operations and lastly the superposition operations.
\begin{lemma}[Normal form for compositions]\label{lem:nmc}
Assume subsampling operations all have a constant
acceptance rate.
A normal form for a composition of subsampling, point
transition and superposition operations is obtained by
applying the following rules until no further can apply.
\begin{eqnarray*}
S^q(S^{q'}(\tilde\mu)) & \rightarrow &
S^{qq'}(\tilde\mu))~, \\
S^{q}(T(\tilde\mu)) & \rightarrow &
T(S^q(\tilde\mu))~, \\
S^q(\tilde\mu \oplus \tilde\mu') &\rightarrow &
 S^q(\tilde\mu) \oplus S^q(\tilde\mu') ~,\\
T(\tilde\mu \oplus \tilde\mu') 
&\rightarrow & T(\tilde\mu) \oplus T(\tilde\mu')~.
\end{eqnarray*}
The remaining top level set of superpositions are
then flattened out by removing any precedence ordering.
\end{lemma}

Note that Lemmas~\ref{theo:PosSup},~\ref{theo:Levysub},~\ref{lem:comp}
and~\ref{lem:nmc} all apply to NRMs as well due to Lemma~\ref{lem:OECR}.
Now it is ready to state the main theorem about the relation between the CRM
and the corresponding NRM under the three dependency operations.

\begin{theorem}[Equivalence Theorem]\label{theo:TDNRM}
  Assume the subsampling rates $q(\cdot)$ are independent (constant)\footnote{This assumption 
  is to deal with the case when considering point transition, meaning we can drop this assumption
  if no point transition operation is considered.} for each point of the corresponding Poisson process, 
  the following dependent random measures~(\ref{eq:normalizedTr}) 
  and (\ref{eq:unnormalizedTr}) are equivalent:
  \begin{itemize}
   \item Manipulate the normalized random measures:
	\begin{equation}\label{eq:normalizedTr}
	 \mu_m^\prime\sim T(S^q(\mu_{m-1}^\prime))\oplus \mu_m, 
~~~~~~~~~~~~~~~~~~\mbox{for } m > 1.
	\end{equation}
  \item Manipulate the completely random measures:
	\begin{eqnarray} \label{eq:unnormalizedTr}
	&&   \tilde{\mu}_m^\prime\sim 
            \tilde{T}(\tilde{S}^q(\tilde\mu_{m-1}^\prime))\oplus \tilde\mu_m, 
~~~~~~~~~~~~~~~~~~\mbox{for }m > 1.
         \nonumber\\
	&& \mu_m^\prime = \frac{\tilde{\mu}_m^\prime}{\tilde{\mu}_m^\prime(\mathbb{X})},  
	\end{eqnarray}
  \end{itemize}
  Furthermore, both resulting NRMs $\mu_m^\prime$'s are equal to:
  \begin{equation} \label{eq:normalizedTr1}
  \mu_m^\prime = \sum_{j = 1}^m
	    \frac{\left(q^{m-j}\tilde{\mu}_j\right)(\mathbb{X})}
	    {\sum_{j^\prime=1}^m\left(q^{m-j^\prime}\tilde{\mu}_{j^\prime}\right)(\mathbb{X})}T^{m-j}(\mu_j), ~~~~~~~~~~~~~~~\mbox{for } m >1 \nonumber
  \end{equation}
  where $q^{m-j}\tilde{\mu}$ is the random measure with L\'{e}vy measure $q^{m-j}(\mathrm{d}x)\nu(\mathrm{d}t, \mathrm{d}x)$,
and $\nu(\mathrm{d}t, \mathrm{d}x)$ is the L\'{e}vy measure of $\tilde{\mu}$.   $T^{m-j}(\mu)$ denotes point transition on $\mu$ for $(m-j)$ times .
\end{theorem}

In the posterior sampling for subsampling operation, we can prove the following posterior 
of the Bernoulli variables.

\begin{theorem}[Posterior acceptance rates for subsampling]\label{theo:posteriorSub}
 Let $\tilde{\mu}^\prime = \sum_kJ_k\theta_k$ be a completely random measure on $\mathbb{X}$,
  $\tilde{\mu} = S^q(\tilde{\mu}^\prime) := \sum_kz_kJ_k\delta_k$ be its subsampling version,
  where $z_k$'s are independent Bernoulli random variables with acceptance rate $q$.
  Further define $\mu = \frac{\tilde{\mu}}{\tilde{\mu}(\mathbb{X})}$. Given $n = \sum_kn_k$ 
  observed data in $\mu$, the posterior of $z_k$ is:
  \begin{equation}\label{eq:postersub}
    p(z_{k}=1|\tilde{\mu}, n) = 
    \begin{cases}
      1 & \text{if } n_k > 0, \\
      \frac{q/J}{q/J + (1 - q)/J^{-k}} & \text{if }n_k = 0.
    \end{cases}
  \end{equation}
  where $J = \left(\sum_{k^\prime}z_{k^\prime}J_{k^\prime}\right)^{n}$, 
  $J^{-k} = \left(\sum_{k^\prime\neq k}z_{k^\prime}J_{k^\prime}\right)^{n}$. 
\end{theorem}

\begin{corollary}[Posterior acceptance rates in sampling $J_{mk}^\prime$ in Section~4~\cite{ChenDB:ICML12}]\label{coro:posterioracc}
 Using the terminology as in Section~4~\cite{ChenDB:ICML12}, 
  the posterior $p(z_{mk}=1|\tilde{\mu}_m,\{\tilde{n}_{mk}^\prime\})$ is computed as:
\begin{itemize}
 \item If $\tilde{n}_{mk}^\prime > 0$, then $p(z_{mk}=1|\tilde{\mu}_m,\{\tilde{n}_{mk}^\prime\})=1$.
  \item Otherwise,
    $$ p(z_{mk}=1|\tilde{\mu}_m,\{\tilde{n}_{mk}^\prime\})=\frac{q^{m - m^\prime}/J_m}{q^{m - m^\prime}/J_m + (1 - q^{m - m^\prime})/J_m^{-k}},$$
    where $J_m = \left(\sum_{m'\le m}\sum_{k^\prime}z_{m' k^\prime}J_{m' k^\prime}\right)^{\tilde{n}_{m\cdot}^\prime}$, 
    $J_m^{-k} = \left(\sum_{m'\le m} \sum_{k^\prime\neq k}z_{m'k^\prime}J_{m'k^\prime}\right)^{\tilde{n}_{m\cdot}^\prime}$, 
    and $\tilde{n}_{m\cdot}^\prime = \sum_{k'} \tilde{n}_{mk'}^\prime$. 
\end{itemize}
\end{corollary}

\paragraph{Acknowledgements.}
NICTA is funded by the Australian Government as represented by the
Department of Broadband, Communications and the Digital Economy and
the Australian Research Council through the ICT Centre of Excellence
program.

\bibliography{citations}
\bibliographystyle{alpha}

\newpage
\appendix

\section{Proofs}\label{sec:proof}

\begin{proof}[Proof of Lemma~\ref{lem:scale}]
We have $\nu(\mathrm{d}x,\mathrm{d}t/\lambda)$.
Doing a change of variables $t'=t/\lambda$ and some rearranging
of the L\'{e}vy-Khintchine formula
yields the following:
\[
\mathbb{E}\left[e^{-\int_{\mathbb{X}} (\lambda f(x)) \left(\tilde\mu(\mathrm{d}x)/\lambda\right)}
 \right] ~=~
e^{-\int_{\mathbb{R}^+ \times \mathbb{X}} \left( 1-e^{-t'\, (\lambda f(x))}\right)
\lambda\nu(\mathrm{d}x,\mathrm{d}t')}
\]
Since $\tilde\mu(\mathrm{ d}x)/\lambda$ normalises to the same
measure as $\tilde\mu(\mathrm{ d}x)$, and saying
something holds for any $f(x)$ is the same as saying
something holds for any $\lambda f(x)$
(when $\lambda>0$),
the result follows.
\end{proof}

\begin{proof}[Proof of Lemma~\ref{lemma:sliceimf}]
  First, for the infinite mixture model, we have infinite number of components, thus
  given the observed data $(x_1, \cdots, x_N)$ and their allocation indicators $\vec{s}$, 
  the model likelihood is
  $$f_{\mu}(\vec{x}, \vec{s}|\vec{\theta}, \vec{J}) = \prod_{i = 1}^N\frac{J_{s_i}}{J_+}g_0(x_i|\theta_{s_i}),$$
  where $J_+ = \sum_{k=1}^\infty J_k$.
  Now introduce the slice auxiliary variables $\vec{u}$ for each data, such that we
  only consider the components whose jumps are larger than a threshold $u_i$ for data $x_i$,
  in this auxiliary space we have
  $$f_{\mu}(\vec{x}, \vec{u}, \vec{s} | \vec{\theta}, \vec{J}) = \frac{1}{J_+^N}\prod_{i = 1}^N 1(u_i < J_{s_i})g_0(x_i|\theta_{s_i}).$$
  Now using the fact that $$\frac{1}{J_+^N} = \frac{\int_0^\infty U_N^{N-1}\exp\left\{-U_NJ_+\right\}\mathrm{d}U_N}{\Gamma(N)},$$
  after introducing the auxiliary variable $U_N$, we have
  $$f_{\mu}(\vec{x}, \vec{u}, \vec{s}, U_N | \vec{\theta}, \vec{J}) \propto U_N^{N-1}\exp\left\{-U_NJ_+\right\} 
    \prod_{i = 1}^N 1(u_i < J_{s_i})g_0(x_i|\theta_{s_i}).$$
  Further decomposing $J_+$ as $$J_+ = J^* + \sum_{k = 1}^K J_k,$$
  where $K$ is the number of jumps which are large than a threshold $L$, $J^* = \sum_{k = K + 1}^\infty J_k$, then we get
  \begin{eqnarray}
    &&f_{\mu}(\vec{x}, \vec{u}, \vec{s}, U_N | \vec{\theta}, J_1, \cdots, J_K, K) \nonumber\\
      &\propto& U_N^{n-1}\exp\left\{-U_N\sum_{k = 1}^KJ_k\right\} 
      \mathbb{E}\left[\exp\left\{-U_NJ^*\right\}\right]\prod_{i = 1}^N 1(u_i < J_{s_i})g_0(x_i|\theta_{s_i}).
  \end{eqnarray}
  Now use the L\'{e}vy-Khintchine representation of a L\'{e}vy process (\ref{eq:levy-khintchine}) 
  to evaluate $\mathbb{E}\left[\exp\left\{-U_NJ^*\right\}\right]$, we get
  \begin{eqnarray}
    &&f_{\mu}(\vec{x}, \vec{u}, \vec{s}, U_N | \vec{\theta}, J_1, \cdots, J_K, K)
      \propto U_N^{N-1}\exp\left\{-U_N\sum_{k = 1}^KJ_k\right\} \nonumber\\ 
      &&\exp\left\{-M\int_0^L\left(1 - \exp\left\{-U_Nt\right\}\right)\rho_\eta(t)\mathrm{d}t\right\}
      \prod_{i = 1}^N 1(u_i < J_{s_i})g_0(x_i|\theta_{s_i}).
  \end{eqnarray}
  Now combining with the priors
  $$p(J_1, \cdots, J_K) = \prod_{k = 1}^K\frac{\rho_\eta(J_k)}{\int_L^\infty\rho_\eta(t)\mathrm{d}t},$$
  $$K\sim\text{Poisson}(M\int_L^\infty\rho_{\eta}(\mathrm{d}t)),\hspace{1cm}\theta_k\sim h(\theta_k),$$
  the result follows.
\end{proof}

\begin{proof}[Proof of Theorem~\ref{thm:postngg}]
The definition for $\tau_n(u)$ comes from [Proposition~1]\cite{JamesLP:SJS09}.
The posterior marginal of Equation~(\ref{eq:postngg})
comes from [Proposition~3]\cite{JamesLP:SJS09}
and is simplified using the change of variables
$t=M\left(1+u\right)^{a}$.
For the predictive posterior, the weights 
in Equation~(\ref{eq:omegadef}) are derived directly from
the posterior.
The posterior proportionality
for $p(U_N=u|\vec{X},\mbox{NGG}(a,M,H(\cdot)))$ 
discards terms not containing $u$.
\end{proof}

\begin{proof}[Proof of Corollary~\ref{cor:pdd}]
Marginalise out $M$ from the posterior
of Equation~(\ref{eq:postngg}) using the alternative 
definition of $T^{N,K}_{a,M}$.
It can be seen this yields 
the posterior of a Poisson-Dirichlet distribution
with discount parameter $a$ and concentration
parameter $b$.
Since the posteriors are equivalent for all data, 
the distributions
are equivalent almost surely.
\end{proof}

\begin{proof}[Proof of Lemma~\ref{lem:eval}]
Equation~(\ref{eq:le1}) holds by noticing
$ 1-\left(\frac{M}{t}\right)^{1/a}\le 1$.
To prove Equation~(\ref{eq:le5}), first prove
\begin{equation}
\label{eq:pr5}
T^{N,K}_{a,M} ~=~
\sum_{n=0}^{N-1} {N-1 \choose n} \left(-M^{1/a}\right)^{n} \Gamma\left(K-\frac{n}{a},M\right)
~~~~~~~~~~~~~~~~~
\forall N,K\in \mathbb{N}^+~,
\end{equation}
This holds by expanding
the term $\left( 1-\left(\frac{M}{t}\right)^{1/a}\right)^{N-1}$
using the binomial expansion and absorbing the powers
$1/t^{n/a}$ into the $t^{K-1}$ as an incomplete Gamma integral.

Now manipulate Equation~(\ref{eq:pr5}).
Expand $\Gamma\left(K-\frac{n}{a},M\right)$ using the
recursion for the incomplete gamma function,
which can be applied when $K-\frac{n}{a} \ne 1$,
yields
\begin{eqnarray*}
&=&\sum_{n=0}^{N-1} {N-1 \choose n} \left(-M^{1/a}\right)^{n}
   \left(
  \left(K-1-\frac{n}{a}\right) \Gamma\left(K-\frac{n}{a},M\right)
  + M^{K-1-\frac{n}{a}} e^{-M} \right)
\\
&=&\sum_{n=0}^{N-1} {N-1 \choose n} \left(-M^{1/a}\right)^{n}
  \left(K-1-\frac{n}{a}\right) \Gamma\left(K-1-\frac{n}{a},M\right)
  + M^{K-1} e^{-M}
      \sum_{n=0}^{N-1} {N-1 \choose n}(-1)^{n}
\end{eqnarray*}
The second sum is a binomial expansion of $(1-1)^{N-1}$ and therefore disappears.
Apply this step repeatedly to get Equation~(\ref{eq:le5}).
Note for the chain of expansions to be done, it must be the case that 
$k-\frac{n}{a}\ne 1$ for $n=0,...,N-1$ and $k=2,...,K$,
so $\frac{n}{a}\ne k$ for $n=0,...,N-1$ and $k=1,...,K-1$,
so $a \ne \frac{n}{k}$ for $n=1,...,N-1$ and $k=1,...,K-1$,
so $k a$ cannot be integral for $k=1,...,K-1$.

Equation~(\ref{eq:le6}) holds by applying the
integration by parts formula on the
terms
$A(t)=\left( 1-\left(\frac{M}{t}\right)^{1/a}\right)^{N-1}$
and $B(t)=t^{K+1}e^{-t}$ and rearranging the resultant integrals
using
$\left(\frac{M}{t}\right)^{1/a} =
1 - \left( 1- \left(\frac{M}{t}\right)^{1/a}\right)$
to arrive back at terms representable.
Note that 
$\left.A(t)B(t)\right|_M^\infty = 0$.
Equation~(\ref{eq:le7}) holds by expanding
\[
\left( 1-\left(\frac{M}{t}\right)^{1/a}\right)^{(N+1)-1}
~=~
\left( 1-\left(\frac{M}{t}\right)^{1/a}\right)^{N-1}
-\left( 1-\left(\frac{M}{t}\right)^{1/a}\right)^{N-1}
 \left(\frac{M}{t}\right)^{1/a}
\]
inside the integral definition of $T_{a,M}^{N+1,K}$.
\end{proof}

\begin{proof}[Proof of Theorem~\ref{thm:cpostngg}]
The posterior marginal of Equation~(\ref{eq:cpostngg})
comes from [Proposition~4]\cite{JamesLP:SJS09}.
Although the denominator is difficult to evaluate,
and it can be derived through a recursion,
the easiest way is simply to normalise
the renumerator.  Sum over 
$\left(Ma\left(1+u\right)^a\right)^K\prod_{k=1}^K  (1-a)_{n_k-1}$ for all length $K$ partitions
$(n_1,n_2,...,n_K)$ yields
$\left(Ma\left(1+u\right)^a\right)^K S^N_{K,a}$
and the result follows by again summing over $K$.
The predictive posterior, as before, follows directly from the
posterior marginal.
The posterior proportionality for $U_N$,
$p(U_N=u|\vec{X},\mbox{NGG}(a,M,H(\cdot)))$,
comes from [Proposition~4]\cite{JamesLP:SJS09} after
discarding terms not containing $u$.
The normalising constant is obtained using the methods of
Theorem~\ref{thm:postngg}.
\end{proof}

\begin{proof}[Proof of Theorem~\ref{thm:grif}]
This comes from  \cite{GriffinW:JCGS11} at the end of Section~3,
and includes the prior on $K_L,J_1,...,J_{K_L}$ described in
Section~4.  The mixture model component
$k(y_i|\theta_{s_i})$ has also been stripped and the slice
sampling variables marginalised out.
\end{proof}

\begin{proof}[Proof of Corollary~\ref{cor:grif}]
Equation~(\ref{eq:corgrif1}) can be seen to hold true
since conditioning it on $U_N=u$ and $\vec{X}$
yields respectively
Equation~(\ref{eq:cpostngg}) and Equation~(\ref{eq:cpostngg:u}).

Prove Equation~(\ref{eq:corgrif2}) as follows.
Now the likelihood 
$p(X_1,...,X_N\,|\,\vec\mu)$ is given by $\frac{1}{T^N}\prod_{k=1}^K J_k^{n_k}$
where $T=\sum_{k=1}^\infty J_k$, the total sum of jumps.
We first simplify this using a latent relative mass variable.
Introduce the variable $U=\gamma/T$ for $\gamma \sim \Gamma(N,1)$,
Adding the term $p(\gamma)\mbox{d}\gamma$ to the likelihood  
and making a change of variable using $U=\gamma/T$,
yields 
\begin{eqnarray*}
p(\gamma)\mbox{d}\gamma&=&\frac{1}{\Gamma(N)} (U T)^{N-1}e^{-U T} T\mbox{d}U\\
&=& T^N \frac{U^{N-1}e^{-U T}}{\Gamma(N)}\mbox{d}U
\end{eqnarray*}
Thus
\[
p(X_1,...,X_N,U\,|\,\vec\mu)
~=~ \frac{U^{N-1}e^{-U T_0}}{\Gamma(N)} \prod_{k=1}^K e^{-u J_k}J_k^{n_k}
\]
where $T_0$ is the total of jumps for the unobserved data.
Now while the prior for the jumps $\rho_{a,M}(t)$
is unnormalised, with observed data it becomes normalised.
Thus 
\begin{eqnarray*}
\lefteqn{p\left(X_1,...,X_N,U,J_1,...,J_K\,|\,T_0,\mbox{NGG}\left(a,M,H(\cdot)\right)\right) ~\propto~}&& \\
&&
 \frac{U^{N-1}e^{-U T_0}}{\Gamma(N)} 
            \prod_{k=1}^K e^{-(1+u) J_k}J_k^{n_k-a-1} H(X^*_k)
\end{eqnarray*}
Taking the expectation of $e^{-U T_0}$ over the remainder term of
the measure $\vec\mu$ corresponds to the 
L\'{e}vy-Khintchine formula, and thus
\begin{eqnarray*}
\lefteqn{p\left(X_1,...,X_N,U,J_1,...,J_K\,|\,\mbox{NGG}\left(a,M,H(\cdot)\right)\right) ~\propto~}&& \\
&&
 \frac{U^{N-1}e^{-M\left((1+u)^a-1\right)}}{\Gamma(N)} 
            \prod_{k=1}^K e^{-(1+u) J_k}J_k^{n_k-a-1} H(X^*_k)
\end{eqnarray*}
Adding in the terms for 
$p\left(X_1,...,X_N,U\,|\,\mbox{NGG}\left(a,M,H(\cdot)\right)\right)$
yields the results and reveals the normalisation constant.
\end{proof}

\begin{proof}[Proof of Lemma~\ref{lemma:covNRM}]
 This uses a similar technique to that of Theorem~1 in~\cite{GriffinKS:11}. 
Using the identity
$1/b = \int_0^\infty e^{-v\,b}\mathrm{d}v$ we get
  \begin{eqnarray}\label{eq:meanNRM}
    &&\mathbb{E}\left[\mu(B)\right] = \mathbb{E}\left[\frac{\tilde{\mu}(B)}{\tilde{\mu}(\mathbb{X})}\right] \nonumber\\
      &=& \int_0^\infty\mathbb{E}\left[\tilde{\mu}(B)\exp\left\{-v\tilde{\mu}(B)\right\}\right]
	    \mathbb{E}\left[\exp\left\{-v\tilde{\mu}(\mathbb{X}\setminus B)\right\}\right]\mathrm{d}v ~.
  \end{eqnarray}

  According to the L\'{e}vy-Khintchine representation of $\tilde{\mu}$ and definition (\ref{eq:laplace-exponent}), we have
  \begin{eqnarray}\label{eq:LevyQ}
    &&\mathbb{E}\left[\exp\left\{-v\tilde{\mu}(B)\right\}\right] = \exp\left\{-P(B)M\tilde\psi_\eta(v)\right\} \label{eq:LevyQ1}\\
    &&\mathbb{E}\left[\tilde{\mu}(B)\exp\left\{-v\tilde{\mu}(B)\right\}\right] = 
	    -\mathbb{E}\left[\frac{\mathrm{d}}{\mathrm{d}v}\exp\left\{-v\tilde{\mu}(B)\right\}\right] \nonumber\\
    &=&P(B)M\tilde\psi_{\eta}^\prime(v)\exp\left\{-P(B)M\tilde\psi_\eta(v)\right\} \label{eq:LevyQ2}\\
    &&\mathbb{E}\left[\tilde{\mu}(B)^2\exp\left\{-v\tilde{\mu}(B)\right\}\right] = 
	    \mathbb{E}\left[\frac{\mathrm{d}}{\mathrm{d}v^2}\exp\left\{-v\tilde{\mu}(B)\right\}\right] \nonumber\\
    &=&\left(P(B)^2M^2\left(\tilde\psi_{\eta}^\prime(v)\right)^2 - P(B)M\tilde\psi_\eta^{\prime\prime}(v)\right)
    \exp\left\{-P(B)M\tilde\psi_\eta(v)\right\} \label{eq:LevyQ3}
  \end{eqnarray}

  Substituting (\ref{eq:LevyQ1}) and (\ref{eq:LevyQ2}) into (\ref{eq:meanNRM}) and using the fact in (\ref{eq:L0}),
  after simplifying we have
  $$\mathbb{E}\left[\mu(B)\right] = P(B).$$

  Since $\text{Var}\left(\mu(B)\right) = \mathbb{E}\left[\mu(B)^2\right] - \left(\mathbb{E}\left[\mu(B)\right]\right)^2$,
  and the last term is equal to $\left(P(B)\right)^2$, we now deal with the first term.
  \begin{eqnarray}\label{eq:varQ}
    &&\mathbb{E}\left[\mu(B)^2\right] = \mathbb{E}\left[\frac{\tilde{\mu}(B)^2}{\tilde{\mu}(\mathbb{X})^2}\right] \nonumber\\
    &=& \int_0^\infty\!\!\int_0^\infty\mathbb{E}\left[\tilde{\mu}(B)^2\times
	      \exp\left\{-v_1\tilde{\mu}(\mathbb{X}) - v_2\tilde{\mu}(\mathbb{X})\right\}\right]
	      \mathrm{d}v_1\mathrm{d}v_2 \\
    &=& \int_0^\infty\!\!\int_0^\infty\mathbb{E}\left[\tilde{\mu}(B)^2\exp\left\{-(v_1 + v_2)\tilde{\mu}(B)\right\}\right]
    \mathbb{E}\left[\exp\left\{-(v_1 + v_2)\tilde{\mu}(\mathbb{X}\setminus B)\right\}\right]\mathrm{d}v_1\mathrm{d}v_2 \nonumber
  \end{eqnarray}

  Substituting (\ref{eq:LevyQ1})(\ref{eq:LevyQ3}) into (\ref{eq:varQ}) we have
  \begin{eqnarray} \label{eq:varQ1}
    (\ref{eq:varQ}) &=& \int_0^\infty\!\!\int_0^\infty \left[P(B)^2M^2\left(\tilde\psi_\eta^\prime(v_1 + v_2)\right)^2
		    - P(B)M\tilde\psi_{\eta}^{\prime\prime}(v_1 + v_2)\right] \nonumber\\
		    &&\exp\left\{-M\tilde\psi_\eta(v_1 + v_2)\right\}\mathrm{d}v_1\mathrm{d}v_2~.
  \end{eqnarray}
  Furthermore, let $v = v_1 + v_2, B = \mathbb{X}$ in (\ref{eq:LevyQ3}), after integrating out $v_1, v_2$ in $[0, \infty]$, we have
  \begin{eqnarray} \label{eq:equation1}
    &&\hspace{-1cm}\int_0^\infty\!\!\int_0^\infty M^2\left(\tilde\psi_\eta^\prime(v_1 + v_2)\right)^2
	  \exp\left\{-M\tilde\psi_\eta(v_1 + v_2)\right\}\mathrm{d}v_1\mathrm{d}v_2 \\
    &\hspace{-0.5cm}=&\hspace{-0.5cm} 1 + \int_0^\infty\!\!\int_0^\infty M\tilde\psi_\eta^{\prime\prime}(v_1 + v_2)
	  \exp\left\{-M\tilde\psi_\eta(v_1 + v_2)\right\}\mathrm{d}v_1\mathrm{d}v_2 \nonumber
  \end{eqnarray}
  Substitute (\ref{eq:equation1}) into (\ref{eq:varQ1}) and simplify we get 
  \begin{eqnarray}
      &&\text{Var}(\mu(B)) = \nonumber\\
      && P(B)(1 - P(B))M
    \int_0^\infty\!\!\int_0^\infty-\tilde\psi_\eta^{\prime\prime}(v_1 + v_2)\exp\left\{-M\tilde\psi_\eta(v_1 + v_2)\right\}\mathrm{d}v_1\mathrm{d}v_2~.
  \end{eqnarray}
  Now use a change of variables, let $v_1^\prime = v_1, v_2^\prime = v_1 + v_2$ and simplify we get the result of (\ref{eq:varNRM}).
\end{proof}

\begin{proof}[Proof of Theorem~\ref{theo:covSuper}]
 From the definition we have
  \begin{eqnarray}
   &&\hspace{-1cm}\text{Cov}\left(\mu_k(B), \mu(B)\right) =
	  \sum_{i=1}^n\text{Cov}\left(\frac{M_i}{\sum_jM_j}\mu_i(B), \mu_k(B)\right) \nonumber\\
    &=& \text{Cov}\left(\frac{M_k}{\sum_jM_j}\mu_k(B), \mu_k(B)\right)
	+ \sum_{i\neq k} \text{Cov}\left(\frac{M_i}{\sum_jM_j}\mu_i(B), \mu_k(B)\right) \label{eq:superCov1} \\
    &=& \mathbb{E}\left[\frac{\tilde{\mu}_k(B)^2}{\left(\sum_j\tilde{\mu}_j(\mathbb{X})\right)\tilde{\mu}_k(\mathbb{X})}\right]
    - \mathbb{E}\left[\frac{\tilde{\mu}_k(B)}{\sum_j\tilde{\mu}_j(\mathbb{X})}\right]\mathbb{E}\left[\frac{\tilde{\mu}_k(B)}{\tilde{\mu}_k(\mathbb{X})}\right] \nonumber\\
    &+& \sum_{i\neq k}\left\{\mathbb{E}\left[\frac{\tilde{\mu}_i(B)\tilde{\mu}_k(B)}{\left(\sum_j\tilde{\mu}_j(\mathbb{X})\right)\tilde{\mu}_k(\mathbb{X})}\right]
    - \mathbb{E}\left[\frac{\tilde{\mu}_i(B)}{\sum_j\tilde{\mu}_j(\mathbb{X})}\right]\mathbb{E}\left[\frac{\tilde{\mu}_k(B)}{\tilde{\mu}_k(\mathbb{X})}\right]\right\} \nonumber
  \end{eqnarray}
  Note that for the Dirichlet process, the last $n-1$ terms of (\ref{eq:superCov1}) vanish because $\mu_i$'s
  are independent from their total mass $M_i$'s, but this is not the case for general NRMs.
  Now we calculate these term by term.

  For the first term, we have
  \begin{eqnarray}\label{eq:cov_first}
    &&\mathbb{E}\left[\frac{\tilde{\mu}_k(B)^2}{\left(\sum_j\tilde{\mu}_j(\mathbb{X})\right)\tilde{\mu}_k(\mathbb{X})}\right] \nonumber\\
    &=& \int_0^\infty\!\!\int_0^\infty\mathbb{E}\left[\tilde{\mu}_k(B)^2\exp\left\{-v_1(\sum_j\tilde{\mu}_j)(\mathbb{X}) - v_2\tilde{\mu}_k(\mathbb{X})\right\}\right]\mathrm{d}v_1\mathrm{d}v_2 \nonumber\\
    &=& \int_0^\infty\!\!\int_0^\infty\mathbb{E}\left[\tilde{\mu}_k(B)^2\exp\left\{-(v_1 + v_2)\tilde{\mu}_k(B)\right\}\right]
    \mathbb{E}\left[\exp\left\{-(v_1 + v_2)\tilde{\mu}_k(\mathbb{X}\setminus B)\right\}\right] \nonumber\\
    &&\mathbb{E}\left[\exp\left\{-v_1(\sum_{j\neq k}\tilde{\mu}_j(\mathbb{X}))\right\}\right]\mathrm{d}v_1\mathrm{d}v_2 \nonumber\\
    &=& \int_0^\infty\!\!\int_0^{v_2}\left(P(B)^2M_k^2\tilde\psi_\eta^\prime(v_1)^2 - P(B)M_k\tilde\psi_\eta^{\prime\prime}(v_1)\right)
    \exp\left\{-M_k\tilde\psi_\eta(v_1)\right\} \nonumber\\
    &&\exp\left\{-(\sum_{j\neq k}M_j)\tilde\psi_\eta(v_2)\right\}\mathrm{d}v_1\mathrm{d}v_2 \nonumber\\
    &=& P(B)M_k\int_0^\infty \gamma(M_k, P(B), v)\exp\left\{-(\sum_{j\neq k}M_j)\tilde\psi_\eta(v)\right\}\mathrm{d}v
  \end{eqnarray}

  For the second term, we have
  \begin{eqnarray}
    \mathbb{E}\left[\frac{\tilde{\mu}_k(B)}{\sum_j\tilde{\mu}_j(\mathbb{X})}\right]\mathbb{E}\left[\frac{\tilde{\mu}_k(B)}{\tilde{\mu}_k(\mathbb{X})}\right]
    &=& P(B)\int_0^\infty\mathbb{E}\left[\tilde{\mu}_k(B)\exp\left\{-v\sum_j\tilde{\mu}_j(\mathbb{X})\right\}\right]\mathrm{d}v \nonumber\\
    &=& P(B)^2M_k\int_0^\infty \tilde\psi_\eta^\prime(v)\exp\left\{-(\sum_jM_j)\tilde\psi_\eta(v)\right\}\mathrm{d}v \nonumber\\
    &=& \frac{P(B)^2M_k\exp\left\{-\left(\sum_jM_j\right)\tilde\psi_\eta(0)\right\}}{\sum_jM_j} \nonumber\\
    &=& \frac{P(B)^2M_k}{\sum_jM_j}
  \end{eqnarray}
  For the third term, similarly
  \begin{eqnarray}\label{eq:cov_third}
    &&\mathbb{E}\left[\frac{\tilde{\mu}_i(B)\tilde{\mu}_k(B)}{\left(\sum_j\tilde{\mu}_j(\mathbb{X})\right)\tilde{\mu}_k(\mathbb{X})}\right] \nonumber\\
    &=& \int_0^\infty\!\!\int_0^\infty\mathbb{E}\left[\tilde{\mu}_i(B)\tilde{\mu}_k(B)\exp\left\{-v_1(\sum_j\tilde{\mu}_j)(\mathbb{X}) - v_2\tilde{\mu}_k(\mathbb{X})\right\}\right]\mathrm{d}v_1\mathrm{d}v_2 \nonumber\\
    &=& \int_0^\infty\!\!\int_0^\infty\mathbb{E}\left[\tilde{\mu}_k(B)\exp\left\{-(v_1 + v_2)\tilde{\mu}_k(B)\right\}\right]
    \mathbb{E}\left[\exp\left\{-(v_1 + v_2)\tilde{\mu}_k(\mathbb{X}\setminus B)\right\}\right] \nonumber\\
    &&\mathbb{E}\left[\tilde{\mu}_i(B)\exp\left\{-v_1\tilde{\mu}_i(B)\right\}\right]
    \mathbb{E}\left[\exp\left\{-v_1\tilde{\mu}_i(\mathbb{X}\setminus B)\right\}\right]\mathbb{E}\left[\exp\left\{-v_1(\sum_{j\neq \{i, k\}}\tilde{\mu}_j(\mathbb{X}))\right\}\right]\mathrm{d}v_1\mathrm{d}v_2 \nonumber\\
    &=& \int_0^\infty\!\!\int_0^\infty P(B)M_k\tilde\psi_\eta^\prime(v_1 + v_2)
    \exp\left\{-M_k\tilde\psi_\eta(v_1 + v_2)\right\} \nonumber\\
    &&P(B)M_i\tilde\psi_\eta^\prime(v_1)
    \exp\left\{-M_i\tilde\psi_\eta(v_1)\right\} \nonumber\\
    &&\exp\left\{-(\sum_{j\neq \{i, k\}}M_j)\tilde\psi_\eta(v_1)\right\}\mathrm{d}v_1\mathrm{d}v_2 \nonumber\\
    &=& P(B)^2M_iM_k\int_0^\infty \tilde\psi_\eta^\prime(v_1)
    \exp\left\{-(\sum_{j\neq k}M_j)\tilde\psi_\eta(v_1)\right\}
    \int_0^{v_1}\tilde\psi_\eta^\prime(v_2)\exp\left\{-M_k\tilde\psi_\eta(v_2)\right\}\mathrm{d}v_2\mathrm{d}v_1 \nonumber\\
    &=& P(B)^2M_i\left(\frac{1}{\sum_{j\neq k}M_j} - \frac{1}{\sum_{j}M_j}\right)\exp\left\{-(\sum_jM_j)\tilde\psi_\eta(0)\right\} \nonumber\\
    &=& P(B)^2M_i\left(\frac{1}{\sum_{j\neq k}M_j} - \frac{1}{\sum_{j}M_j}\right)
  \end{eqnarray}
  The fourth term is similar to the second term, and is equal to
  \begin{eqnarray}\label{eq:fourthterm}
    \mathbb{E}\left[\frac{\tilde{\mu}_i(B)}{\sum_j\tilde{\mu}_j(\mathbb{X})}\right]\mathbb{E}\left[\frac{\tilde{\mu}_k(B)}{\tilde{\mu}_k(\mathbb{X})}\right]
    &=& \frac{P(B)^2M_i\exp\left\{-\left(\sum_jM_j\right)\tilde\psi_\eta(0)\right\}}{\sum_jM_j} \nonumber\\
    &=& \frac{P(B)^2M_i}{\sum_jM_j}
  \end{eqnarray}
  The result follows.
\end{proof}

\begin{proof}[Proof of Theorem~\ref{theo:covSub}]
 By subsampling, we obtain two independent NRMs $\mu^q$ and $\mu_0^q$, corresponding to those points selected
  and those rejected by the independent Bernoulli trials, respectively. 

  We denote the total mass of the corresponding unnormalized $\mu^q$ as $M_q$, 
  and $M_q^0$ for $\mu_0^q$. From the definition of subsampling, we have
  $$M_q := (q\tilde{\mu})(\mathbb{X}) = \int_{\mathbb{X}}q(x)\tilde{\mu}(x)\mathrm{d}x,$$
  $$M_q^0 = M - M_q.$$

  Furthermore, notice that
  the original NRM $\mu$ is the superposition of $\mu^q$ and $\mu_0^q$. Thus according to Theorem~\ref{theo:covSuper},
  the covariance between $\mu$ and $\mu^q$ is
  $$P(B)M_q\int_0^\infty \gamma(M_q, P(B), v)
      \exp\left\{-(M - M_q)\tilde\psi_\eta(v)\right\}\mathrm{d}v
   + P(B)^2\left(\frac{2M_q - M}{M}\right)~,$$
\end{proof}

\begin{proof}[Proof of Theorem~\ref{theo:covPoiTrS}]
  Note that $\tilde{\mu}$ and $\tilde{\mu}^\prime$ 
  are not independent, thus they can not be separated
  when taking the expectation. Now let $A$ and $B$ are
  defined as in the theorem, then:
  \begin{eqnarray}\label{eq:covQ}
    &&\mathbb{E}\left[\mu(B)\left((T\mu)(B)\right)\right] = \mathbb{E}\left[\frac{\tilde{\mu}(B)}{\tilde{\mu}(\mathbb{X})}
	      \frac{\tilde{\mu}^\prime(B)}{\tilde{\mu}^\prime(\mathbb{X})}\right]
	      = \mathbb{E}\left[\frac{\tilde{\mu}(B)}{\tilde{\mu}(\mathbb{X})}
	      \frac{\tilde{\mu}(A)}{\tilde{\mu}(\mathbb{X})}\right] \nonumber\\
    &=&\int_0^\infty\!\!\int_0^\infty\mathbb{E}\left[\tilde{\mu}(B)\tilde{\mu}(A)\times
	      \exp\left\{-(v_1 + v_2)\tilde{\mu}(\mathbb{X})\right\}\right]
	      \mathrm{d}v_1\mathrm{d}v_2 \nonumber\\
    &=&\int_0^\infty\!\!\int_0^\infty\mathbb{E}\left[\tilde{\mu}(B)\exp\left\{-(v_1 + v_2)\tilde{\mu}(B)\right\}\right] \nonumber\\
	      &&\hspace{1.3cm}\mathbb{E}\left[\tilde{\mu}(A)\exp\left\{-(v_1 + v_2)\tilde{\mu}(A)\right\}\right] \nonumber\\
	      &&\hspace{1.3cm}\mathbb{E}\left[\exp\left\{-(v_1 + v_2)\tilde{\mu}(X/\{A\cup B\})\right\}\right]\mathrm{d}v_1\mathrm{d}v_2 \nonumber\\
    &=&\int_0^\infty\!\!\int_0^\infty P(B)M\tilde\psi_\eta^\prime(v_1 + v_2)\exp\left\{-P(B)M\tilde\psi_\eta(v_1 + v_2)\right\} \nonumber\\
	      &&\hspace{1.3cm}P(A)M\tilde\psi_\eta^\prime(v_1 + v_2)\exp\left\{-P(A)M\tilde\psi_\eta(v_1 + v_2)\right\} \nonumber\\
	      &&\hspace{1.3cm}P(\mathbb{X}/\{A\cup B\})M\tilde\psi_\eta^\prime(v_1 + v_2)\exp\left\{-P(\mathbb{X}/\{A\cup B\})M\tilde\psi_\eta(v_1 + v_2)\right\}\mathrm{d}v_1\mathrm{d}v_2 \nonumber\\
    &=&P(A)P(B)M^2\int_0^\infty\!\!\int_0^{v_1}\tilde\psi_\eta^\prime(v_2)^2\exp\left\{-M\tilde\psi_\eta(v_2)\right\}\mathrm{d}v_2\mathrm{d}v_1 \nonumber
  \end{eqnarray}
  Then the covariance is:
  \begin{eqnarray}
    &&\text{Cov}\left(\mu(B), (T\mu)(B)\right) \nonumber\\
    &=& \mathbb{E}\left[\mu(B)\left((T\mu)(B)\right)\right] - \mathbb{E}\left[\mu(B)\right]\mathbb{E}\left[(T\mu)(B)\right] \nonumber\\
    &=& P(A)P(B) \nonumber\\
    &&\left(M^2\int_0^\infty\int_0^{v_1} \tilde\psi_{\eta}^\prime(v_2)^2\exp\left\{-M\tilde\psi_\eta(v_2)\right\}\mathrm{d}v_2\mathrm{d}v_1 - 1\right)
  \end{eqnarray}
\end{proof}

\begin{proof}[Proof of Lemma~\ref{theo:PosSup}]
  From the existing of Poisson processes, each L\'{e}vy measure $\nu_i(\mathrm{d}t, \mathrm{d}x)$ corresponds
  to a Poisson random measure $N_i(\mathrm{d}t, \mathrm{d}x)$ with 
  $$\mathbb{E}\left[N_i(\mathrm{d}t, \mathrm{d}x)\right] = \nu_i(\mathrm{d}t, \mathrm{d}x).$$ Also we have $\forall i$,
  $$\tilde{\mu}_i(\mathrm{d}x) = \int_{0}^\infty tN_i(\mathrm{d}t, \mathrm{d}x).$$
  Thus from (\ref{eq:superp}) we have $$\tilde{\mu}(\mathrm{d}x) = \int_{0}^\infty t\left(\sum_{i = 1}^n
  N_i(\mathrm{d}t, \mathrm{d}x)\right) = \int_{0}^\infty tN(\mathrm{d}t, \mathrm{d}x),$$
  where $N(\cdot) = \sum_{i = 1}^nN_i(\cdot)$ is again a Poisson random measure. Thus the L\'{e}vy intensity for $\tilde{\mu}(\cdot)$ is
  \begin{eqnarray}
    \nu(\mathrm{d}t, \mathrm{d}x) &=& \sum_{i = 1}^n\nu_i(\mathrm{d}t, \mathrm{d}x).
  \end{eqnarray}
  Because Theorem~1 in~\cite{JamesLP:SJS09} applies for any CRMs with L\'{e}vy measure $\nu(\mathrm{d}t, \mathrm{d}x)$,
  the Lemma is proved.
\end{proof}

\begin{proof}[Proof of Lemma~\ref{theo:Levysub}]
This follows by merging the impact of the
subsampling operation with the sampling step in
Lemma~\ref{lem:sampler}.
Suppose the L\'evy measure is in the form
$M\rho(\mathrm{d}t|x) H(\mathrm{d}x)$.
The infinitesimal rate at data point $x_i$ 
when sampling the jump is now
$q(x_i)M\rho(\mathrm{d}t|x)$.
Thus the L\'evy measure for the subsampled measure must be 
$M\rho(\mathrm{d}t|x) q(x)H(\mathrm{d}x)$.

This argument can be seen from the detailed derivation below.
 First note that $S^q(\tilde{\mu})$ is equivalent to
  \begin{equation}
    S^q(\tilde{\mu}) = \int_{R^+\times\mathbb{X}}z(\mathrm{d}x)sN(\mathrm{d}s, \mathrm{d}x).
  \end{equation}
  Let $B\in\mathbb{X}$, we divide $B$ into $n$ non-overlap patches and use $A_{nm}$
  to denote the $m$-th patch of them. So we have
  \begin{eqnarray}\label{eq:levysub}
    \mathbb{E}_{N(\cdot), z}\left[e^{-uS^q(\tilde{\mu})(B)}\right] &\stackrel{n\rightarrow\infty}{=}& \mathbb{E}_{N(\cdot), z}\left[e^{-\sum_{A_{nm}\in B}uz(A_{nm})s_{nm}N(A_{nm}, s_{nm})}\right] \nonumber\\
    &=& \mathbb{E}_{N(\cdot), z}\left[\prod_{A_{nm}\in B}e^{-uz(A_{nm})s_{nm}N(A_{nm}, s_{nm})}\right] \nonumber\\
    &=& \prod_{A_{nm}\in B}\mathbb{E}_{N(\cdot), z}\left[e^{-uz(A_{nm})s_{nm}N(A_{nm}, s_{nm})}\right] \nonumber\\
    &=& e^{\sum_{A_{nm}\in B}\log\left\{\mathbb{E}_{N(\cdot), z}\left[e^{-uz(A_{nm})s_{nm}N(A_{nm}, s_{nm})} - 1\right] + 1\right\}} \nonumber\\
    &\stackrel{(a)}{=}& e^{\sum_{A_{nm}\in B}\mathbb{E}_{N(\cdot), z}\left[e^{-uz(A_{nm})s_{nm}N(A_{nm}, s_{nm})} - 1\right]} \nonumber\\
    &\stackrel{(b)}{=}& e^{q\sum_{A_{nm}\in B}\mathbb{E}_{N(\cdot)}\left[e^{-us_{nm}N(A_{nm}, s_{nm})} - 1\right]} \nonumber\\
    &\stackrel{n\rightarrow\infty}{=}& e^{-\int_{R^+\times B}\left(1 - e^{-us}\right)\left(q\nu(\mathrm{d}s, \mathrm{d}x)\right)} \nonumber\\
  \end{eqnarray}
  Here $(a)$ above follows because $\mathbb{E}_{N(\cdot)}\left[\left(e^{-uz(A_{nm})s_{nm}N(A_{nm}, s_{nm})} - 1\right)\right]$ is infinitesimal 
  thus $\log(1 + x)\stackrel{x\rightarrow 0}{\sim} x$ applies. $(b)$ is obtained by integrating out $z(A_{nm})$
  with Bernoulli distribution. Thus it can be seen from (\ref{eq:levysub}) that $S^q(\tilde{\mu})$ has the L\'{e}vy measure of 
  $q(\mathrm{d}x)\nu(\mathrm{d}t, \mathrm{d}x)$.
\end{proof}

\begin{proof}[Proof of Theorem~\ref{theo:TDNRM}]
  We show that starting from (\ref{eq:unnormalizedTr}) and (\ref{eq:normalizedTr}), we can both end up 
  the random measures defined in (\ref{eq:normalizedTr1}).

  First, for the operations in (\ref{eq:unnormalizedTr}), adapting from Theorem~2.17 of~\cite{Cinlar:10}, 
  a Poisson random measure with mean measure $\nu$ 
  on the space $\mathbb{R}^+\times\mathbb{X}$ has the form
  \begin{equation}\label{eq:poisson1}
    N = \sum_{n=1}^\infty\sum_{i < K_n}\delta_{(s, x)},
  \end{equation}
  where $K_n$ is a Poisson distributed random variable with mean $\nu$, and 
  $(s\in\mathbb{R}^+, x\in\mathbb{X})$ are points in the corresponding
  Poisson processes. Then a realization of $N$ composes of points in a Poisson process $\Pi_1$,
  and the corresponding Poisson random measure can be written as 
  $N_1 = \sum_{(s, x)\in\Pi_1}\delta_{(s, x)}$.

  Now consider doing a subsampling $S^q$ and a point transition $T$ on
  $\Pi_1$, by the definitions and (\ref{eq:poisson1})
  we get a new random measure
  \begin{eqnarray}
    \tilde{N} &=& T(S_q(N_1)) = T(S_q(\sum\delta_{(s, x)})) \nonumber\\
    &\stackrel{\text{(*)}}{=}& \sum z(q(T(x)))\delta_{(s, T(x))} \stackrel{\text{(**)}}{=} \sum z(q(x))\delta_{(s, T(x))},
  \end{eqnarray}
  where $z(q(\cdot))$ means a Bernoulli random variable with acceptance rate $q(\cdot)$, $(*)$ follows from definitions,
  $(**)$ follows from the assumption of constant subsampling rate.

  It is easy to show by induction that by subsampling and point transitioning $i$ times of the Poisson process $\Pi_1$, 
  we get a random measure as
  \begin{equation}
    \tilde{N}^\prime = \sum z(q^i(x))\delta_{(s, T^i(x))}.
  \end{equation}

  By the definition, when superpositioning the this Poisson process $T^i(S^q_i(\Pi_1))$ with another 
  Poisson process $\Pi_2$ with mean measure $\nu_2$, we get another random measure as
  \begin{equation}
    N^{\prime\prime} = \sum_{(s, x)\in\Pi_1} z(q^i(x))\delta_{(s, T^i(x))} 
	      + \sum_{(s, x)\in\Pi_2}\delta_{(s, T(x))}.
  \end{equation}

  This Poisson random measure is then used to construct a completely random measure $\tilde{\mu}$ using (\ref{eq:RM}) as:
  \begin{eqnarray}
  &&\tilde{\mu}(A) = \int_{\mathbb{R}^+\times\mathbb{X}}\!\!
		sN^{\prime\prime}(\mathrm{d}s, \mathrm{d}x) \nonumber\\
    &=& \sum_{(s, x)\in\Pi_1} z(q^i(x))s\delta_{(s, T^i(x))} + \sum_{(s, x)\in\Pi_2} 
		s\delta_{(s, x)}.
  \end{eqnarray}

  By marginalize over $r$'s and normalizing this random measure, we get
  \begin{eqnarray} \label{eq:NRMST}
  &&\mu(A) = \frac{\tilde{\mu}(A)}{\tilde{\mu}(\mathbb{X})} \nonumber\\
    &=& \frac{M_1^\prime}{M_1 + M_2^\prime}\frac{\sum_{(s, x)\in \Pi_1\cap A} s\delta_{(s, T^i(x))}}
	{\sum_{(s, x)\in \Pi_1\cap \mathbb{X}} s\delta_{(s, T^i(x))}} \nonumber\\ 
    &+& \frac{M_2^\prime}{M_1^\prime + M_2^\prime}\frac{\sum_{(s, x)\in \Pi_2\cap A} s\delta_{(s, T^i(x))}}
	{\sum_{(s, x)\in \Pi_2\cap \mathbb{X}} s\delta_{(s, T^i(x))}} \nonumber\\
    &=& \frac{M_1^\prime}{M_1^\prime + M_2^\prime}(T^i\mu_1)(A) + \frac{M_2^\prime}{M_1^\prime + M_2^\prime}(T^i\mu_2)(A),\nonumber\\
  \end{eqnarray}
  where by apply Lemma~\ref{theo:Levysub} we conclude that $M_1^\prime = \left(q^i\tilde{\mu}_1\right)(\mathbb{X})$ is the total mass of the random measure 
  with L\'{e}vy measure $q^j(\mathrm{d}x)\nu(\mathrm{d}t, \mathrm{d}x)$ and $M_2^\prime = \tilde{\mu}_2(\mathbb{X})$. 
  We use the fact that $(T^k\tilde{\mu}_i)(\mathbb{X}) = \tilde{\mu}_i(\mathbb{X})$ in the derivation of (\ref{eq:NRMST}),
  because the point transition operation only moves the points $(s, x)$ of the Poisson process 
  to other locations $(s, x + \mathrm{d}x)$, thus does not affect the total mass of the corresponding 
  random measure.

  This means by superpositioning after subsampling, the mass of the normalized random measure decades 
  exponentially fast with respect to the distance $i$. Based on Eq.~(\ref{eq:NRMST}), when taking $i$ 
  from 1 to $n$, and taking superposition for all these random measure induced, the resulting normalized 
  random measure is:
  \begin{equation}
    \mu_n^\prime = \sum_{i = 1}^n\frac{\left(q^{n-i}\tilde{\mu}_i\right)(\mathbb{X})}{\sum_{j=1}^n\left(q^{n-j}\tilde{\mu}_j\right)(\mathbb{X})}T^{n-i}(\mu_i).
  \end{equation}

  Next, for the operations in (\ref{eq:normalizedTr}), from the definition we have
  \begin{eqnarray}
    \mu_2^\prime &=& T\left(S^q\left(\mu_1^\prime\right)\right)\oplus\mu_2 \nonumber\\
      &=& \frac{\left(q\tilde{\mu}_1\right)(\mathbb{X})}{\left(q\tilde{\mu}_1 + \tilde{\mu}_2\right)(\mathbb{X})}T\left(\mu_1\right)
	+ \frac{\left(\tilde{\mu}_1\right)(\mathbb{X})}{\left(q\tilde{\mu}_1 + \tilde{\mu}_2\right)(\mathbb{X})}\mu_2
  \end{eqnarray}
  Now $\mu_2^\prime$ has a total mass of $(q\tilde{\mu}_1 + \tilde{\mu}_2)(\mathbb{X})$, by induction on $i$, 
  we get the formula in (\ref{eq:normalizedTr1}) for $i = n$.

  This completes the proof.
\end{proof}

\begin{proof}[Proof of Theorem~\ref{theo:posteriorSub}]
  Given the current data configuration \{$n_k, k = 1, 2, \cdots$\}, for a particular $k$, 
  \begin{itemize}
   \item If $n_k > 0$,
      this means this jump $J_k$ must exist in $\mu$, otherwise it is impossible to have $n_k > 0$,
      thus $p(z_{k}=1|\tilde{\mu}, n) = 1$.

   \item Otherwise,since $\mu = \sum_{k:z_k = 1}\frac{J_k\delta_k}{\sum_{k^\prime}z_{k^\prime}J_{k^\prime}}$,
      we have the likelihood as:
      $$\prod_{{k^{\prime\prime}}:n_{k^{\prime\prime}}>0}\frac{J_{k^{\prime\prime}}^{n_{k^{\prime\prime}}}}{(\sum_{k^\prime\neq k}z_{k^\prime}J_{k^\prime} + z_kJ_k)^{n_k}}
	= \frac{\prod_{{k^{\prime\prime}}:n_{k^{\prime\prime}}>0}J_{k^{\prime\prime}}^{n_{k^{\prime\prime}}}}{(\sum_{k^\prime\neq k}z_{k^\prime}J_{k^\prime} + z_kJ_k)^{n}}.$$
      Furthermore, we know that the prior for $z_k$ is $p(z_k = 1) = q$, thus the posterior is:
      $$p(z_k = 1|\tilde{\mu}, n)\propto \frac{q}{(\sum_{k^\prime\neq k}z_{k^\prime}J_{k^\prime} + J_k)^{n}}.$$
      $$p(z_k = 0|\tilde{\mu}, n)\propto \frac{1 - q}{(\sum_{k^\prime\neq k}z_{k^\prime}J_{k^\prime})^{n}}.$$
      After normalizing, we get the posterior for the case $n_k = 0$ in (\ref{eq:postersub}).
  \end{itemize}
\end{proof}

\begin{proof}[Proof of Corollary~\ref{coro:posterioracc}]
  Note that $J_{mk}^\prime$ is obtained by subsampling of \{$J_{m^\prime k}, m^\prime\leq m$\},
  the number of data points in $\tilde{\mu}_m^\prime$ is denoted as $\tilde{n}_{m\cdot}^\prime = \sum_{k^\prime}\tilde{n}_{mk^\prime}^\prime$.

  Following the same arguments as in the proof of Theorem~\ref{theo:posteriorSub}, when $\tilde{n}_{mk}^\prime > 0$,
  $p(z_{mk}=1|\tilde{\mu}_m,\tilde{n}_{m\cdot}')=1$. Otherwise, by subsampling, $\mu_m^\prime$ can be written as:
  $$\mu_m^\prime = \sum_{m^\prime\leq m}\sum_{k^\prime:z_{m^\prime k^\prime} = 1}
    \frac{z_{m^\prime k^\prime}J_{m^\prime k^\prime}\delta_{\theta_{m^\prime k^\prime}}}
    {\sum_{m^{\prime\prime}\leq m}\sum_{k^{\prime\prime}}z_{m^{\prime\prime} k^{\prime\prime}}J_{m^{\prime\prime} k^{\prime\prime}}}.$$
  Now following the same proof of Theorem~\ref{theo:posteriorSub}, if we define
  $$J_m = \left(\sum_{m'\le m}\sum_{k^\prime}z_{m' k^\prime}J_{m' k^\prime}\right)^{\tilde{n}_{m\cdot}^\prime},
  J_m^{-k} = \left(\sum_{m'\le m} \sum_{k^\prime\neq k}z_{m'k^\prime}J_{m'k^\prime}\right)^{\tilde{n}_{m\cdot}^\prime}~,$$
  then we get the likelihood as 
  $$\frac{\prod_{{k^{\prime\prime}}:\tilde{n}_{mk^{\prime\prime}}^\prime>0}{J_{mk^{\prime\prime}}^\prime}^{n_{k^{\prime\prime}}}}{J_m}.$$
  Furthermore, from subsampling, we know that the Bernoulli prior for $z_{mk}$ is $q^{m - m^\prime}$, and the posterior
  can then be derived using the Bayes rule as in the proof of Theorem~\ref{theo:posteriorSub}.
\end{proof}

\end{document}